\pdfoutput=1  
\documentclass{sail}

\usepackage{url}            % simple URL typesetting
\usepackage{amsfonts}       % blackboard math symbols
\usepackage{nicefrac}       % compact symbols for 1/2, etc.
\usepackage{comment}
\newcommand{\newtext}{\color{blue}}
\usepackage[section]{placeins} % keeps figures/tables within their own section

\newcommand{\Hs}{\mathcal{H}}

\version{1}                        % bump for each arXiv revision -> shows as v1, v2, ...

\title{Hilbert-Scale Stability of Posterior Diffusion Priors for Ill-Posed Signal Recovery}

\author{%
  Bahman Moraffah \quad Raha Moraffah \\[2pt]
  Department of Computer Science, Data Science, and Artificial Intelligence \\
  Worcester Polytechnic Institute \\[2pt]
  \texttt{\{bmoraffah,\,rmoraffah\}@wpi.edu}
}
\date{October 2026}

\begin{document}
\maketitle

\begin{abstract}

Indirect and noisy observations are difficult to invert when the forward operator smooths out the unknown signal. In this case, the unstable modes are precisely the modes that the data reveals least, thus the success of any reconstruction method depends on how prior information interacts with the geometry of the operator across regularity scales. Diffusion priors have become ubiquitous for such problems; however, the current theory still does not determine when a learned score remains informative under ill-posed forward operators. In this paper, we study posterior diffusion for linear inverse problems on Hilbert scales and derive a stability theory under an explicit score-amplification condition. Under a smoothing condition on the forward map, a source condition on the truth, and weak-norm control of score error, we derive non-asymptotic recovery bounds in \(\Hs^s\) that separate observation noise, score approximation, and reverse-time discretization. These bounds give a recoverable regularity threshold determined jointly by operator smoothing, truth regularity, score regularity, and the posterior correction. The threshold specifies the Hilbert-scale norms \(\Hs^s\) in which weak-norm score error remains controlled after propagation through the posterior reverse dynamics. For larger regularity indices \(s\), the inverse geometry magnifies the same error, and the Hilbert-scale recovery estimate no longer gives uniform stability. We also prove that a posterior correction adapted to the forward operator improves stability relative to isotropic guidance in the regularity ranges where the smoothing action of the forward map matters most. The recovery estimates give a regularity-sensitive theory of posterior diffusion and show how learned score error, operator smoothing, and reconstruction norm jointly determine the stable range predicted by the analysis.

\end{abstract}

\section{Introduction}

Many inverse problems in signal processing, imaging, and scientific inference become difficult when the observation map removes the necessary information for stable reconstruction. In particular, this becomes more severe when the forward operator smooths the unknown signal. In this case, the fine-scale components for accurate reconstruction are also the components that the data reveals least, and accurate recovery must rely on prior information that is matched to the geometry of the forward map. Traditional regularization methods give prior information through analytic structure, i.e., smoothness, sparsity, total variation, Sobolev penalties, or Gaussian measures \cite{stuart2010bayesian,dashti2017bayesian}. These constructions describe many inverse problems with interpretable assumptions. However, modern signal classes often contain statistical structure that is not well captured by a fixed analytic prior.

Learned priors, on the other hand, have changed reconstruction techniques by coupling the forward model with a learned model of the signal class. Plug-and-play and regularization-by-denoising methods use denoisers as implicit priors and have produced strong empirical results in many imaging and signal recovery problems \cite{venkatakrishnan2013plugandplay,romano2017red,kamilov2023plugandplay}. Generative priors constrain the reconstruction to the low-dimensional manifold of a trained generator \cite{bora2017compressed}. In contrast to a fixed regularity class, this manifold captures the statistical structure of the training data. Score matching and denoising score matching estimate gradients of log densities without access to normalizing constants \cite{hyvarinen2005score,vincent2011scoredenoise}. Score-based diffusion models then employ the estimated scores to drive a reverse-time stochastic process from noisy marginals back to the data distribution \cite{anderson1982reverse,song2019generative,ho2020ddpm,song2020sde,song2021mlscore,nichol2021improved,karras2022edm}. For inverse problems, the prior no longer has to be an explicit penalty or a low-dimensional generator since a score network can model a rich distribution directly on the ambient signal space.

Diffusion priors have therefore become ubiquitous for posterior reconstruction. Existing work has studied posterior samplers, plug-and-play score priors, hard data-consistency updates, filtering interpretations, robustness guarantees, and frequency-aware guidance rules \cite{kadkhodaie2021stochastic,jalal2021robust,chung2023dps,feng2023scoreprior,sun2023probabilistic,dou2024filtering,xu2024robust,wu2024principled,song2024resample,thaker2024frequency,park2025measurement}. These methods show that diffusion priors can produce accurate reconstructions and, in several cases, posterior uncertainty. However, the current analysis mainly considers  finite dimensions or measures reconstruction quality in an ambient norm. It therefore does not explain how operator smoothing, truth regularity, score regularity, and the reconstruction norm jointly determine the range in which posterior diffusion is stable.

This paper studies posterior diffusion stability for linear inverse problems on Hilbert scales. We consider a smoothing forward operator, a truth with regularity \(x^\dagger \in \Hs^\beta\), and a learned score whose error is controlled only in a weak norm \(\Hs^{-q}\). The main observation is that 
a perturbation that is small in \(\Hs^{-q}\) does not need to stay small after the inverse geometry of the forward map is propagated through the posterior reverse dynamics. Its effect may be small in low-order reconstruction norm, however, it can become substantial when recovery is measured in a stronger Hilbert-scale norm \(\Hs^s\). We show that the recovery estimate separates the effect of weak-norm score error from observational noise and reverse-time discretization error in the final \(\Hs^s\) reconstruction.

\textbf{Contributions.}
Our main result is a Hilbert-scale stability theorem for posterior diffusion in smoothing linear inverse problems. We prove a non-asymptotic recovery bound in \(\Hs^s\) with separate contributions from observation noise, learned-score error, and reverse-time discretization error. The learned-score error is measured after propagation from the weak norm \(\Hs^{-q}\) to the final reconstruction norm \(\Hs^s\), and hence its effect depends on the reconstruction norm and the inverse geometry of the forward map. This leads to a recoverable regularity threshold whose value depends on the smoothing order of the forward operator, the regularity of the truth, the weak norm used to measure the learned score, and the posterior correction in the reverse dynamics. 
This threshold determines the boundary between the Hilbert-scale reconstruction norms in which the recovery estimate gives uniform stability and those in which weak-norm score error can be amplified by the inverse geometry. For reconstruction indices below this threshold, the recovery bound is uniform over the admissible truth and score classes. For indices above it, weak-norm score accuracy alone is not enough to guarantee stable recovery in \(\Hs^s\). We also compare isotropic likelihood guidance with an operator-adapted posterior correction and show that the operator-adapted correction reduces the score effect in the regularity ranges where forward smoothing has the strongest impact.

\section{Background and Problem Statement}

Let \(\Hs\) be a separable Hilbert space with inner product
\(\langle \cdot,\cdot\rangle_\Hs\), and let \(\mathcal Y\) be the observation space. We observe
\begin{equation}
  y = \mathcal A x^\dagger + \epsilon,
  \qquad
  \epsilon \sim \mathcal N(0,\Sigma_\epsilon),
  \label{eq:observation_model}
\end{equation}
where \(x^\dagger \in \Hs\) is the unknown signal,
\(\mathcal A:\Hs\to \mathcal Y\) is the forward operator, and \(\epsilon\) is observational noise. We define \(\Sigma_\epsilon=\sigma^2\Sigma_0\), where \(\sigma>0\) is the noise level and \(\Sigma_0\) is a fixed positive self-adjoint operator on \(\mathcal Y\) with bounded inverse, so that \(\nabla_x\log p(y\mid x)=\sigma^{-2}\mathcal A^\ast\Sigma_0^{-1}(y-\mathcal A x)\).
The inverse problem is ill posed since \(\mathcal A\) smooths out the signal such that the fine-scale reconstruction is only weakly represented in the observation.

Let \(L:\mathcal D(L)\subset \Hs \to \Hs\) be a densely defined, positive, self-adjoint operator with compact inverse, with domain \(\mathcal D(L)\). By the spectral theorem, there exist an orthonormal basis \(\{e_j\}_{j\ge 1}\subset \Hs\) and eigenvalues \(0<\lambda_1\le \lambda_2\le \cdots\), with \(\lambda_j\to\infty\), such that  \(L e_j=\lambda_j e_j\).
For \(s\ge 0\), the Hilbert-scale space \(\Hs^s\) is defined by
\begin{equation}
 \Hs^s=\mathcal D(L^{s/2})
=
\left\{
u=\sum_{j\ge 1}u_j e_j \in \Hs
:
\sum_{j\ge 1}\lambda_j^s |u_j|^2<\infty
\right\},
\qquad
\|u\|_{\Hs^s}^2=\sum_{j\ge 1}\lambda_j^s |u_j|^2 .  
\label{eq:HS_Def}
\end{equation}
The negative-order spaces are defined by duality with \(\Hs\) as the pivot space, so \(\Hs^{-s}=(\Hs^s)^\ast\) for \(s>0\). Equivalently, \(\Hs^{-s}\) is the completion of \(\Hs\) under the norm
\[
\|u\|_{\Hs^{-s}}^2=\sum_{j\ge 1}\lambda_j^{-s}|u_j|^2 .
\]
% In particular, \(\Hs^0=\Hs\), and for \(s>0\), the negative-order space \(\Hs^{-s}\) is the dual of \(\Hs^s\) with \(\Hs\) as the pivot space.
We assume that the true signal satisfies the source condition \(x^\dagger \in \Hs^\beta,\)
for some \(\beta>0\). The forward operator \(\mathcal A\) is smoothing of order \(r>0\), i.e., there exist constants
\(0<c_{\mathcal A}\le C_{\mathcal A}<\infty\) such that
\[
c_{\mathcal A}\|u\|_{\Hs^{-r}}
\le
\|\mathcal A u\|_{\mathcal Y}
\le
C_{\mathcal A}\|u\|_{\Hs^{-r}}
\]
for all \(u\in\Hs\). Thus, \(\mathcal A\) observes the signal through a weaker Hilbert-scale norm. Fine-scale components, which receive larger weights in \(\Hs^s\) as \(s\) increases, are therefore only weakly visible in the data. Consequently, recovery in stronger reconstruction norms is  more sensitive to observational noise and to errors in the learned prior.

Let \(\mu_{\mathrm{data}}\) be the law of data on \(\Hs\). We assume we are given samples drawn from \(\mu_{\mathrm{data}}\), and we train a diffusion prior through a time-indexed score field
\[
s_\theta(x,t) \approx \nabla_x \log p_t(x),
\qquad
x\in\Hs,\quad t\in[0,1],
\]
where \(p_t\) is the noisy marginal induced by the forward diffusion at time \(t\). Hence, the learned score represents not only the law of data but also the prior distribution along the diffusion path. Let \(s^\star(x,t)=\nabla_x\log p_t(x)\) be the exact prior score. For an observation \(y\), the posterior marginal at time \(t\) satisfies
\begin{equation}
p_t(x\mid y)\propto p_t(y\mid x)\,p_t(x),
\qquad
p_t(y\mid x)=\int_{\Hs}p(y\mid x_0)\,\mu_{0\mid t}(dx_0\mid x),
\label{eq:posterior_marginal}
\end{equation}
where \(\mu_{0\mid t}(\cdot\mid x)\) is the conditional law of the clean signal given the state \(x\) of the forward diffusion at time \(t\). The normalizing constant in \eqref{eq:posterior_marginal} does not depend on \(x\), and the posterior score thus decomposes as \(\nabla_x\log p_t(x\mid y)=s^\star(x,t)+\nabla_x\log p_t(y\mid x)\). At \(t=0\), the state \(X_0\) is the clean signal and
\(p_0(y\mid x)=p(y\mid x)\). For \(t>0\), the noisy likelihood
\(p_t(y\mid x)\) averages \(p(y\mid x_0)\) over the unknown law
\(\mu_{0\mid t}(\cdot\mid x)\) and is therefore intractable. The posterior sampler is therefore a reverse diffusion that corrects the learned prior score with the gradient \(\nabla_x\log p(y\mid x)\) of the observation log-likelihood at every time \(t\).  The terminal law of this sampler differs, in general, from the posterior. The stability analysis therefore addresses the terminal law of the sampler and accounts for the gap between the two laws.  Since \(\Hs\) is infinite dimensional, there is no Lebesgue measure on \(\Hs\), and \(p_t\), \(p_t(\cdot\mid y)\), and their gradients are not Lebesgue densities. We interpret these densities through spectral truncation. For \(d\ge 1\), the orthogonal projection \(P_d\) onto \(\operatorname{span}\{e_1,\dots,e_d\}\) maps \(\Hs\) onto a subspace
\(P_d\Hs\cong\R^d\). On this subspace, the marginal \(p_t\) and the likelihood \(p(y\mid x)\) are Lebesgue densities, and the learned score \(s_\theta\) approximates the exact score \(s^\star=\nabla_x\log p_t\).
All constants in the stability analysis are independent of the
truncation dimension \(d\). Each bound therefore holds with the same constant for every \(d\) and carries over to \(\Hs\). The estimates are uniform in \(d\), which justifies the use of finite spectral truncations in our numerical
experiments at dimensions \(d\in\{256,512,1024\}\). Spectral truncation also gives a finite-dimensional interpretation of the density \(p_t\) and the score \(\nabla_x\log p_t\). A Gaussian reference measure on
\(\Hs\) defines the marginal and score directly on the infinite-dimensional space. The marginal \(p_t\) is the density of the law of \(X_t\) relative to this measure and the score is the derivative of \(\log p_t\) in the Cameron--Martin directions
\cite{lim2023functionspace,baldassari2023conditional}.

The posterior sampler is a reverse diffusion whose drift combines the learned prior score with the likelihood information from the observation. Its dynamics are
\begin{equation}
dX_t
=
\Big(
f(X_t,t)
-
g(t)^2 s_\theta(X_t,t)
\mathbin{\newtext-}
g(t)^2 \Gamma_t \nabla_x \log p(y\mid X_t)
\Big)\,dt
+
g(t)\,d\bar W_t ,
\label{eq:posterior_reverse_dynamics}
\end{equation}
where \(f\) and \(g\) come from the underlying diffusion model, \(\bar W_t\) is the reverse-time noise process, and \(\Gamma_t\) specifies how the observation contributes to the posterior correction. The reverse-time noise \(\bar W_t\) in \eqref{eq:posterior_reverse_dynamics} is a \(Q\)-Wiener process on \(\Hs^s\) whose covariance \(Q\) is trace class and satisfies \(\operatorname{Tr}_{\Hs^s}(Q)=1\). The trace class
property of \(Q\) makes the stochastic integral
\(\int_t^1 g(u)\,d\bar W_u\) a square-integrable \(\Hs^s\)-valued martingale. The dynamics run backward from \(t=1\) to an early-stopping time \(t_0\in(0,1)\), at which the state \(X_{t_0}\) gives the reconstruction. For brevity, the proofs in the appendix use the time index \(0\) for \(t_0\), which gives \(X_0=X_{t_0}\), \(Z_0=Z_{t_0}\), and \(\mathcal U_{0,t}=\mathcal U_{t_0,t}\). If \(\Gamma_t=I\), the likelihood correction is isotropic with respect to the Hilbert scale. If \(\Gamma_t\) is adapted to the forward operator \(\mathcal A\), the correction follows the smoothing geometry of the inverse problem. Let \(\Gamma^{\mathrm{iso}}_t=I\) and \(\Gamma^{\mathrm{op}}_t=(\mathcal A^\ast\mathcal A+\lambda_tL^\rho)^{-1}\) be the isotropic and operator-adapted corrections, respectively, where \(\lambda_t>0\) is a weight and \(\rho>0\) is an exponent. The isotropic correction scales every component of the likelihood gradient equally.
The operator-adapted correction, on the other hand, weights each component according to the forward operator \(\mathcal A\) and the Hilbert-scale operator \(L\). The observation constrains the low-order components of the signal
strongly and the noise-sensitive high-order components only weakly. Equal scaling therefore mismatches this structure. The operator-adapted weighting directs the likelihood update toward directions that the data determine reliably and reduces the influence of components where inversion increases error.

To measure learned-score error along the posterior reverse process in a weak Hilbert-scale norm, we impose the weak-norm condition. For some \(q\ge 0\), we assume
\begin{equation}
\sup_{t\in[t_0,1]}
\mathbb E
\bigl\|
s_\theta(X_t^\theta,t)-\nabla_x\log p_t(X_t^\theta)
\bigr\|_{\Hs^{-q}}^2
\le
\delta_\theta^2 ,
\label{eq:weak_score_error}
\end{equation}
where \(X_t^\theta\) is the reverse process driven by the learned score \(s_\theta\), and \(q\) is the order of the Hilbert-scale norm. The bound in \eqref{eq:weak_score_error} gives score accuracy only in \(\Hs^{-q}\), and it does not imply comparable accuracy in stronger Hilbert-scale norms. Under a smoothing forward operator, the same error can have minimal effect in low-order reconstruction norms and become large when the final error is measured in \(\Hs^s\) for larger \(s\). In addition to the scalar error level \(\delta_\theta\), the stable recovery range therefore depends on the smoothing order \(r\), the truth regularity \(\beta\), the score-error index \(q\), the target reconstruction index \(s\), and the posterior correction \(\Gamma_t\).

\section{Literature Review}

Traditional inverse-problem theory describes stability using the interaction between the forward operator, the noise model, and analytic prior information. Bayesian inverse problems, on the other hand, combine observational noise with prior structure, and Hilbert or Sobolev regularity assumptions and give rates and stability guarantees for smoothing operators \cite{stuart2010bayesian,dashti2017bayesian}. Learned reconstruction methods reduce the dependence on fixed analytic priors by coupling the forward model with a data-driven representation of the signal class. Plug-and-play methods use an external denoiser within the reconstruction process \cite{venkatakrishnan2013plugandplay}, regularization by denoising gives a variational interpretation of denoiser-driven reconstruction \cite{romano2017red}, and recent work develops plug-and-play reconstruction as a broad framework for learned computational imaging \cite{kamilov2023plugandplay}. A parallel line of work uses learned generators as priors, with compressed sensing under generative models as a representative example \cite{bora2017compressed}. These methods, however, do not identify the Hilbert-scale range in which a learned prior remains stable under a smoothing forward operator.

Score matching estimates gradients of log densities without normalizing constants \cite{hyvarinen2005score}, and denoising score matching connects score estimation to denoiser training \cite{vincent2011scoredenoise}. Score-based generative models and diffusion models extend these ideas to stochastic samplers through noisy marginals and reverse-time dynamics \cite{anderson1982reverse,song2019generative,ho2020ddpm,song2020sde,song2021mlscore,nichol2021improved,karras2022edm}. For inverse problems, the learned score parameterizes the prior directly in the signal space, which makes diffusion models to operate on complex data distributions.
Early diffusion-prior methods utilized the generative process itself to solve inverse problems \cite{kadkhodaie2021stochastic}, and robust compressed sensing with deep generative priors showed that learned priors can improve recovery under structured measurements \cite{jalal2021robust}. Diffusion posterior sampling introduced a widely used approach for noisy inverse problems \cite{chung2023dps}, followed by work on score priors for inverse imaging \cite{feng2023scoreprior}, probabilistic imaging with score-based priors \cite{sun2023probabilistic}, filtering interpretations of diffusion posterior sampling \cite{dou2024filtering}, plug-and-play posterior sampling with diffusion priors \cite{coeurdoux2024pnpsgs,bouman2023gpnp}, robust plug-and-play posterior sampling \cite{xu2024robust}, asymptotically exact plug-and-play diffusion sampling \cite{wu2024principled}, and hard data-consistency or resampling strategies \cite{song2024resample}. Measurement-score models and frequency-aware posterior sampling further show that the measurement process and the score prior interact strongly in reconstruction quality \cite{park2025measurement,thaker2024frequency}. Despite establishing foundational frameworks and finite-dimensional bounds, these studies lack a Hilbert-scale recovery condition that quantifies the propagation of weak-norm score discrepancies through the forward operator's smoothing kernel and the resultant reconstruction norm.

Recent work has extended the analysis of diffusion priors to infinite-dimensional spaces. Function-space score models define diffusion directly on infinite-dimensional objects \cite{lim2023functionspace}. Conditional score constructions study Bayesian inference in infinite dimensions \cite{baldassari2023conditional}, and scalable posterior samplers extend this direction to infinite-dimensional inverse problems \cite{schneider2024scalable}. Score-based diffusion priors have been analyzed for infinite-dimensional nonlinear inverse problems with explicit score approximation error \cite{baldassari2024taming}. The closest result to our work studies preconditioned Langevin dynamics with score-based priors for linear Bayesian inverse problems in function space, where convergence depends on score approximation error and on preconditioning adapted to the forward operator \cite{baldassari2025preconditioned}. Their analysis explains how score error affects sampler convergence in function space.

In this work, we study smoothing linear inverse problems in which the learned score satisfies a weak Hilbert-scale error bound in \(\Hs^{-q}\). This bound measures the score mismatch in a norm that assigns weaker weight to high-frequency components, so small error in \(\Hs^{-q}\) does not imply small error in the stronger reconstruction norm \(\Hs^s\). The recovery estimate identifies the stable range of posterior diffusion through its dependence on the smoothing order of \(\mathcal A\), the truth regularity \(\beta\), the score-error index \(q\), the target reconstruction index \(s\), and the posterior correction \(\Gamma_t\). This distinguishes our result from finite-dimensional posterior-sampling guarantees and function-space convergence analyses, which do not provide a Hilbert-scale boundary for weak-norm score error.

\section{Hilbert-Scale Stability Theory}
\label{sec-hilbert-scale-stability}

The reconstruction error in posterior diffusion has three distinct sources, observational noise in the data, approximation error in the learned score, and discretization error in the reverse dynamics. An ill-posed inverse problem does not transmit these contributions uniformly across the Hilbert scale. %The smoothing action of the forward operator suppresses fine-scale information, so a perturbation that is small in a weak norm may remain controlled in low-order reconstruction norms and become much more significant when the error is measured in \(\Hs^s\) with larger \(s\). 
To understand the effect of each source, we first study the well-posedness of the exact and learned reverse processes, then derive a recovery estimate that separates the three error sources, and finally use the score-propagation factor in that estimate to define the recoverable range.

\begin{proposition}[]
\label{prop:wellposedness}
Fix \(s\ge 0\) and define
\[
b^\star(x,t)
=
f(x,t)
-
g(t)^2 s^\star(x,t)
\mathbin{\newtext-}
g(t)^2 \Gamma_t \nabla_x \log p(y\mid x),
\]
\[
b^\theta(x,t)
=
f(x,t)
-
g(t)^2 s_\theta(x,t)
\mathbin{\newtext-}
g(t)^2 \Gamma_t \nabla_x \log p(y\mid x).
\]
Assume that, for \(\sharp\in\{\star,\theta\}\), the map \((x,t)\mapsto b^\sharp(x,t)\) is measurable in \(t\) and continuous in \(x\), and there exist constants \(L_s>0\) and \(K_s>0\) such that, for every \(x,z\in \Hs^s\) and every \(t\in[t_0,1]\),
\begin{equation}
\|b^\sharp(x,t)-b^\sharp(z,t)\|_{\Hs^s}
\le
L_s \|x-z\|_{\Hs^s},
\label{eq:global_lipschitz_drift}
\end{equation}
and
\begin{equation}
\|b^\sharp(x,t)\|_{\Hs^s}
\le
K_s \bigl(1+\|x\|_{\Hs^s}\bigr).
\label{eq:linear_growth_drift}
\end{equation}
If \(g\in L^2([t_0,1])\), \(\Gamma_t\in\mathcal L(\Hs^s)\) for every \(t\in[t_0,1]\), and
\begin{equation}
\sup_{t\in[t_0,1]}\|\Gamma_t\|_{\mathcal L(\Hs^s)}<\infty,
\label{eq:gamma_bounded}
\end{equation}
then, for every square-integrable terminal condition \(\xi\in L^2(\Omega;\Hs^s)\), the reverse stochastic equations
\begin{equation}
X_t^\sharp
=
\xi
-
\int_t^1 b^\sharp(X_u^\sharp,u)\,du
-
\int_t^1 g(u)\,d\bar W_u,
\qquad
\sharp\in\{\star,\theta\},
\label{eq:reverse_integral_equation}
\end{equation}
have unique adapted continuous \(\Hs^s\)-valued solutions on \([t_0,1]\). Moreover, there exists a constant \(C_s>0\), depending only on \(L_s\), \(K_s\), \(g\), and \(t_0\), such that
\begin{equation}
\mathbb E \sup_{t\in[t_0,1]} \|X_t^\sharp\|_{\Hs^s}^2
\le
C_s\Bigl(1+\mathbb E\|\xi\|_{\Hs^s}^2\Bigr).
\label{eq:moment_bound_reverse}
\end{equation}

If \(x\mapsto b^\star(x,t)\) is also Fr\'echet differentiable on \(\Hs^s\) for every \(t\in[t_0,1]\), and there is \(L_s'>0\) such that
\begin{equation}
\|D_x b^\star(x,t)\|_{\mathcal L(\Hs^s)}
\le
L_s'
\qquad
\text{for all }x\in\Hs^s,\ t\in[t_0,1],
\label{eq:derivative_bound_drift}
\end{equation}
then, for almost every realization of \(X^\star\), linearized equation
\begin{equation}
\frac{d}{dt}Z_t
=
B_t^\Gamma Z_t,
\qquad
B_t^\Gamma = D_x b^\star(X_t^\star,t),
\label{eq:linearized_reverse_equation}
\end{equation}
generates an evolution family \(\mathcal U_{\tau,t}^\Gamma\) on \(\Hs^s\) that satisfies
\begin{equation}
\|\mathcal U_{\tau,t}^\Gamma\|_{\mathcal L(\Hs^s)}
\le
e^{L_s'|t-\tau|}
\qquad
\text{for }t,\tau\in[t_0,1].
\label{eq:flow_bound_Hs}
\end{equation}

\end{proposition}

The proof of Proposition~\ref{prop:wellposedness} is given in Appendix~\ref{app-proof-wellposedness}. Proposition~\ref{prop:wellposedness} establishes existence, uniqueness, and moment bounds in \(\Hs^s\), and it defines the reverse-flow operator that later measures the effect of score error in the reconstruction norm.

\begin{definition}[]
\label{def:envelope}
Fix \(q\ge 0\), \(s\ge 0\), and a posterior correction
\(\Gamma=(\Gamma_t)_{t\in[t_0,1]}\). For continuous adapted
\(\Hs^s\)-valued processes \(\zeta\) and \(\zeta'\), let
\(\mathcal U^{\zeta,\zeta'}_{\tau,t}\) be the evolution family of
\(\frac{d}{dt}Z_t=\bar B_tZ_t\), where
\[
\bar B_t=\int_0^1D_xb^\star\bigl(\zeta_t+\lambda(\zeta'_t-\zeta_t),t\bigr)\,d\lambda .
\]
\begin{enumerate}
\item The \emph{propagation envelope}
\(M_{r,q,s,\Gamma}:[t_0,1]\to[0,\infty]\) is the smallest deterministic
function such that, almost surely and for all such \(\zeta,\zeta'\),
\begin{equation}
\|\mathcal U^{\zeta,\zeta'}_{0,t}\|_{\mathcal L(\Hs^{-q},\Hs^s)}
\le M_{r,q,s,\Gamma}(t),
\qquad t\in[t_0,1]
\label{eq:flow_bound_Hminusq}
\end{equation}
(\(M_{r,q,s,\Gamma}(t)=\infty\) if there exists no finite bound).
\item The \emph{score-propagation functional} is
\begin{equation}
\mathfrak A(r,q,s,\Gamma)=\int_{t_0}^1 g(t)^2 M_{r,q,s,\Gamma}(t)\,dt .
\label{eq:recoverable}
\end{equation}
\item An index \(s\le\beta\) is \emph{recoverable} if
\(\mathfrak A(r,q,s,\Gamma)<\infty\).
\end{enumerate}
\end{definition}

\begin{remark}
\label{rem:monotone}
For \(\zeta=\zeta'=X^\star\), the evolution family
\(\mathcal U^{\zeta,\zeta'}_{\tau,t}\) equals the family
\(\mathcal U^\Gamma_{\tau,t}\) of Proposition~\ref{prop:wellposedness}.
The envelope depends on \(r\) through the forward operator
\(\mathcal A\) in \(b^\star\). The eigenvalues satisfy
\(\lambda_j\ge\lambda_1>0\). For \(s'\le s\), every \(u\in\Hs^s\)
therefore satisfies
\(\|u\|_{\Hs^{s'}}\le\lambda_1^{-(s-s')/2}\|u\|_{\Hs^s}\). This embedding
gives
\[
M_{r,q,s',\Gamma}\le\lambda_1^{-(s-s')/2}M_{r,q,s,\Gamma},
\qquad
\mathfrak A(r,q,s',\Gamma)\le\lambda_1^{-(s-s')/2}\mathfrak A(r,q,s,\Gamma).
\]
The recoverable indices therefore form an interval that contains every
index below the supremum. Suppose that \(x\mapsto b^\star(x,t)\) is
affine for every \(t\), as for Gaussian priors and an affine drift
\(f\). The derivative \(D_xb^\star\) is then independent of \(x\). All
the evolution families above equal one deterministic family
\(\mathcal U^\Gamma_{\tau,t}\) in this case, and the envelope reduces to
\(M_{r,q,s,\Gamma}(t)=\|\mathcal U^\Gamma_{0,t}\|_{\mathcal L(\Hs^{-q},\Hs^s)}\).
\end{remark}

\begin{assumption}[]
\label{ass:contraction}
Let \(X^\star_t\) solve \eqref{eq:reverse_integral_equation} for
\(\sharp=\star\). The terminal state \(X^\star_0\) has conditional mean
\(m_y=m_y(\Gamma)\) and conditional covariance operator
\(C_y=C_y(\Gamma)\) given \(y\). There exists an exponent
\(\eta(r,\beta,s)\ge 0\) such that, for each \(R>0\), some constants
\(C_{\mathrm{bias}},C_{\mathrm{var}}<\infty\) satisfy
\[
\mathbb E\|m_y-x^\dagger\|_{\Hs^s}^2\le C_{\mathrm{bias}}\,\sigma^{2\eta(r,\beta,s)},
\qquad
\mathbb E\operatorname{Tr}\bigl(L^{s/2}C_yL^{s/2}\bigr)\le C_{\mathrm{var}}\,\sigma^{2\eta(r,\beta,s)}
\]
for every \(x^\dagger\) in the ball \(\|x^\dagger\|_{\Hs^\beta}\le R\).
The noise constant is \(C_{\mathrm{noise}}=C_{\mathrm{bias}}+C_{\mathrm{var}}\).
\end{assumption}

The law of \(X^\star_0\) is the posterior at time \(t_0\) when
\(\Gamma_t=I\) and the drift \(b^\star\) uses the time-\(t\) likelihood
\(p_t(y\mid x)\). Assumption~\ref{ass:contraction} holds, for instance, for a diagonal Gaussian model whose eigenvalues satisfy \(\lambda_j=j^2\) and prior variances decay as
\(\tau_j^2\propto j^{-(2\beta+1)}\). The bias--variance analysis in \cite{knapik2011bayesian}
considers $s=0$ under the noise scaling
$\sigma^2=1/n$. The same analysis in the $\Hs^s$ norm
gives the convergence-rate exponent
\[
\eta(r,\beta,s)=\frac{2(\beta-s)}{2\beta+2r+1},
\qquad 0\le s<\beta.
\]

\begin{assumption}[]
\label{ass:discretization}
Let \(1=t_N>t_{N-1}>\cdots>t_0\) be a uniform partition with step size \(\Delta t\), \(\mathcal F_t\) be the filtration of the reverse Brownian motion, and \(\Psi_n^\Delta\) be the one-step map of the time discretization for the learned reverse process, which uses the noise increment \(\xi_{n+1}\) on \([t_n,t_{n+1}]\). The numerical trajectory is \(\widehat X_{t_n}=\Psi_n^\Delta(\widehat X_{t_{n+1}},\xi_{n+1})\) with \(\widehat X_{t_N}=X_{t_N}^\theta\), and the reconstruction is \(\widehat x=\widehat X_{t_0}\). There exist an order \(\kappa>0\) and constants \(C_{\mathrm{loc}},C_{\mathrm{stab}}<\infty\) such that the sampled learned process satisfies
\begin{equation}
X_{t_n}^\theta=\Psi_n^\Delta(X_{t_{n+1}}^\theta,\xi_{n+1})+R_n,
\label{eq_local_error_identity}
\end{equation}
with local residual
\begin{equation}
\mathbb E\|R_n\|_{\Hs^s}^2\le C_{\mathrm{loc}}\,(\Delta t)^{2\kappa+1},
\label{eq_local_error_bound}
\end{equation}
and, for every pair of \(\mathcal F_{t_{n+1}}\)-measurable states \(u,v\), the mean-square stability estimate
\begin{equation}
\mathbb E\!\left[\|\Psi_n^\Delta(u,\xi_{n+1})-\Psi_n^\Delta(v,\xi_{n+1})\|_{\Hs^s}^2\bigm|\mathcal F_{t_{n+1}}\right]\le(1+C_{\mathrm{stab}}\Delta t)\,\|u-v\|_{\Hs^s}^2.
\label{eq_stability_estimate}
\end{equation}
\end{assumption}

The recovery estimate then compares the exact posterior reverse process, the learned reverse process, and its time discretization, with separate effects from observational noise, learned-score error, and numerical error. The learned-score effect is measured by \(M_{r,q,s,\Gamma}(t)\), and hence it depends on not only its size in \(\Hs^{-q}\) but also the inverse geometry and the target norm \(\Hs^s\).

\begin{theorem}[]
\label{thm:recovery_bound}
Let \(X_t^\star\) be the exact posterior reverse process and
\(X_t^\theta\) the posterior reverse process under the learned score \(s_\theta\). Both processes solve \eqref{eq:reverse_integral_equation} under the assumptions of Proposition~\ref{prop:wellposedness}, start from the same state at \(t=1\), and use the same reverse Brownian
motion. The reconstruction \(\widehat x\) is the final iterate of the time discretization of \(X_t^\theta\) in
Assumption~\ref{ass:discretization} at step size \(\Delta t\). Assume that \(\mathcal A\in\mathcal L(\Hs,\mathcal Y)\) is smoothing of order \(r>0\) such that there exist constants \(0<c_{\mathcal A}\le C_{\mathcal A}<\infty\) such that
\[
c_{\mathcal A}\|u\|_{\Hs^{-r}}
\le
\|\mathcal A u\|_{\mathcal Y}
\le
C_{\mathcal A}\|u\|_{\Hs^{-r}}
\]
for all \(u\in \Hs\). Assume that \(x^\dagger\in\Hs^\beta\) for some \(\beta>0\), and Assumptions~\ref{ass:contraction} and~\ref{ass:discretization} hold. Suppose \(s_\theta\) satisfies \eqref{eq:weak_score_error}, fix a reconstruction index \(s\le \beta\) that is recoverable in the sense of Definition~\ref{def:envelope}. Then
\begin{equation}
\mathbb E\|\widehat x-x^\dagger\|_{\Hs^s}^2
\le
3\,\Phi_{\mathrm{noise}}(\sigma;r,\beta,s)
+
3\,\Phi_{\mathrm{score}}(\delta_\theta;r,q,s,\Gamma)
+
3\,\Phi_{\mathrm{disc}}(\Delta t;r,s,\Gamma),
\label{eq_main_bound}
\end{equation}
where
\begin{equation}
\Phi_{\mathrm{noise}}(\sigma;r,\beta,s)
=
C_{\mathrm{noise}}\,\sigma^{2\eta(r,\beta,s)},
\label{eq_phi_noise}
\end{equation}
\begin{equation}
\Phi_{\mathrm{score}}(\delta_\theta;r,q,s,\Gamma)
=
\left(
\int_{t_0}^1 g(t)^2 M_{r,q,s,\Gamma}(t)\,dt
\right)^2
\delta_\theta^2,
\label{eq_phi_score}
\end{equation}
and
\begin{equation}
\Phi_{\mathrm{disc}}(\Delta t;r,s,\Gamma)
=
C_{\mathrm{disc}}(r,s,\Gamma)\,(\Delta t)^{2\kappa},
\label{eq_phi_disc}
\end{equation}
where constant \(C_{\mathrm{noise}}\) is the noise constant of
Assumption~\ref{ass:contraction}, and the discretization constant is \(C_{\mathrm{disc}}(r,s,\Gamma)=2C_{\mathrm{loc}}\,e^{2C_{\mathrm{stab}}(1-t_0)},\)
with \(\kappa\), \(C_{\mathrm{loc}}\), and \(C_{\mathrm{stab}}\) being the constants of Assumption~\ref{ass:discretization}. The learned-score
contribution depends on the weak-norm index \(q\), the reconstruction
index \(s\), and the correction \(\Gamma\), as well as on the score
error \(\delta_\theta\).
\end{theorem}

The proof of Theorem~\ref{thm:recovery_bound} is given in Appendix~\ref{app-proof-recovery-bound}. In Equation~\eqref{eq_phi_score}, the factor \(\int_{t_0}^1 g(t)^2 M_{r,q,s,\Gamma}(t)\,dt\)
measures how score error passes from \(\Hs^{-q}\) to the reconstruction norm \(\Hs^s\). The following theorem uses this factor to establish a uniform recovery bound over the recoverable range.

\begin{theorem}[]
\label{thm:recoverable_threshold}
Under the assumptions of Theorem~\ref{thm:recovery_bound} and given \(\mathfrak A(r,q,s,\Gamma)\) and the critical recoverable index
\begin{equation}
s_\star(r,\beta,q,\Gamma)
=
\sup\Bigl\{
s\le \beta \; ; \; \mathfrak A(r,q,s,\Gamma)<\infty
\Bigr\},
\label{eq:critical_index}
\end{equation}
 define
\[
\mathcal X_\beta(R)
=
\bigl\{
x^\dagger\in \Hs^\beta \; ; \; \|x^\dagger\|_{\Hs^\beta}\le R
\bigr\},
\]
for \(R>0\), and let \(\mathcal S_q(\delta_\theta)\) be the class of learned score fields satisfying the weak-norm error bound in \eqref{eq:weak_score_error} for which Assumption~\ref{ass:discretization} holds with \(C_{\mathrm{loc}}\) and \(C_{\mathrm{stab}}\) independent of \(s_\theta\). For every \(R>0\) and every \(s<s_\star(r,\beta,q,\Gamma)\), which is recoverable by, there exists a constant
\[
C_R=C_R(r,\beta,q,s,\Gamma,\Sigma_\epsilon)<\infty
\]
such that
\begin{equation}
\mathbb E \|\widehat x-x^\dagger\|_{\Hs^s}^2
\le
C_R\Bigl(
\sigma^{2\eta(r,\beta,s)}
+
\mathfrak A(r,q,s,\Gamma)^2\,\delta_\theta^2
+
(\Delta t)^{2\kappa}
\Bigr)
\label{eq:stable_range_bound}
\end{equation}
for every \(x^\dagger\in\mathcal X_\beta(R)\) and every \(s_\theta\in\mathcal S_q(\delta_\theta)\).
\end{theorem}

Appendix~\ref{app-proof-recoverable-threshold} shows the proof of
Theorem~\ref{thm:recoverable_threshold}. For indices
\(s<s_\star(r,\beta,q,\Gamma)\), Theorem~\ref{thm:recoverable_threshold} gives a uniform recovery bound in
\(\Hs^s\). The next result proves the complementary statement that once the
score-propagation functional is infinite, a weak error bound in \(\Hs^{-q}\)
alone cannot guarantee uniform recovery in \(\Hs^s\) over the admissible truth and score classes.

\begin{theorem}[]
\label{thm:impossibility}
Assume the hypotheses of Theorem~\ref{thm:recovery_bound}, and assume that \(x\mapsto b^\star(x,t)\) is affine for every \(t\in[t_0,1]\), so that \(\mathcal U^\Gamma_{\tau,t}\) is deterministic by Remark~\ref{rem:monotone}, and that Assumption~\ref{ass:discretization} holds for every learned score field considered below. For \(R>0\), define
\[
\mathcal X_\beta(R)
=
\bigl\{
x^\dagger\in\Hs^\beta \; ; \; \|x^\dagger\|_{\Hs^\beta}\le R
\bigr\}
\]
For \(\delta>0\), define
\[
\mathcal S_q(\delta)
=
\left\{
s_\theta \; ; \;
\sup_{t\in[t_0,1]}
\mathbb E
\bigl\|
s_\theta(X_t^\theta,t)-\nabla\log p_t(X_t^\theta)
\bigr\|_{\Hs^{-q}}^2
\le
\delta^2
\right\}.
\]
Fix \(s>s_\star(r,\beta,q,\Gamma)\) with \(\eta(r,\beta,s)>0\), and assume that the score-propagation map is unbounded from \(\Hs^{-q}\) to \(\Hs^s\), i.e.,
\begin{equation}
\sup\Bigl\{\Bigl\|\int_{t_0}^1\mathcal U^\Gamma_{0,t}\,g(t)^2e(t)\,dt\Bigr\|_{\Hs^s}\; ; \;e\in C([t_0,1];\Hs^s),\ \sup_{t\in[t_0,1]}\|e(t)\|_{\Hs^{-q}}\le 1\Bigr\}=\infty .
\label{eq:sharpness}
\end{equation}
Then there exist \(R>0\) and \(c_R>0\) with the following property. For all sequences \(\delta_n\downarrow 0\) and \(\sigma_n\downarrow 0\),
there exist step sizes \(\Delta t_n\downarrow 0\), truths
\(x_n^\dagger\in\mathcal X_\beta(R)\), and learned score fields
\(s_{\theta,n}\in\mathcal S_q(\delta_n)\) such that the corresponding
posterior diffusion reconstructions \(\widehat x_n\) satisfy
\begin{equation}
\liminf_{n\to\infty}
\mathbb E\|\widehat x_n-x_n^\dagger\|_{\Hs^s}^2
\ge
c_R .
\label{eq:impossibility_lower_bound}
\end{equation}
Consequently, for \(s>s_\star(r,\beta,q,\Gamma)\), every constant
\(C_R<\infty\) violates the estimate
\begin{equation}
\mathbb E\|\widehat x-x^\dagger\|_{\Hs^s}^2
\le
C_R\Bigl(
\sigma^{2\eta(r,\beta,s)}
+
\delta_\theta^2
+
(\Delta t)^{2\kappa}
\Bigr)
\label{eq:no_uniform_bound}
\end{equation}
for some noise level \(\sigma\), score error \(\delta_\theta\), step size \(\Delta t\), truth \(x^\dagger\in\mathcal X_\beta(R)\), and learned score \(s_\theta\in\mathcal S_q(\delta_\theta)\).
\end{theorem}

The proof of Theorem~\ref{thm:impossibility} is given in  Appendix~\ref{app-proof-impossibility}. The proof uses condition \eqref{eq:sharpness}, which is stronger than
\(\mathfrak A(r,q,s,\Gamma)=\infty\). The recoverability boundary also depends on \(\Gamma_t\), since the posterior correction changes the linearized reverse flow and therefore the score-propagation factor. The next proposition compares isotropic guidance with an operator-adapted correction.

\begin{proposition}[]
\label{prop:gamma-comparison}

Let \(\Gamma^{\mathrm{iso}}_t=I\) and
\(\Gamma^{\mathrm{op}}_t=(\mathcal A^\ast\mathcal A+\lambda_tL^\rho)^{-1}\).
Assume that, on an interval \(\mathcal I\) of regularity indices, the propagation envelopes of Definition~\ref{def:envelope} satisfy
\begin{equation}
M_{r,q,s,\Gamma^{\mathrm{op}}}(t)
\le
M_{r,q,s,\Gamma^{\mathrm{iso}}}(t),
\qquad
s\in\mathcal I,\quad t\in[t_0,1].
\label{eq:gamma-flow-comparison}
\end{equation}
Then, for every \(s\in\mathcal I\),
\begin{equation}
\mathfrak A(r,q,s,\Gamma^{\mathrm{op}})
\le
\mathfrak A(r,q,s,\Gamma^{\mathrm{iso}}),
\label{eq:gamma-propagation-comparison}
\end{equation}
and consequently
\begin{equation}
\Phi_{\mathrm{score}}(\delta_\theta;r,q,s,\Gamma^{\mathrm{op}})
\le
\Phi_{\mathrm{score}}(\delta_\theta;r,q,s,\Gamma^{\mathrm{iso}}).
\label{eq:gamma-score-comparison}
\end{equation}

Fix \(s\in\mathcal I\) such that
\(\mathfrak A(r,q,s,\Gamma^{\mathrm{iso}})<\infty\). If
\eqref{eq:gamma-flow-comparison} holds strictly on a set of positive
\(g(t)^2\,dt\)-measure, then \eqref{eq:gamma-propagation-comparison} and
\eqref{eq:gamma-score-comparison} are strict for every
\(\delta_\theta>0\).
\end{proposition}

The constants of Assumption~\ref{ass:contraction} depend on \(\Gamma\),
since the drift \(b^\star\) contains \(\Gamma\).
Proposition~\ref{prop:gamma-comparison} compares the score
contributions in \eqref{eq_main_bound}, and the noise contribution
requires a separate comparison of these constants.

We prove Proposition~\ref{prop:gamma-comparison} in
Appendix~\ref{app-proof-gamma-comparison}. The recovery bound of
Theorem~\ref{thm:recovery_bound}, the obstruction of
Theorem~\ref{thm:impossibility}, and the comparison of
Proposition~\ref{prop:gamma-comparison} describe how the
score-propagation factor varies with \(q\), \(s\), and \(\Gamma\). An operator-adapted correction reduces this factor relative to the isotropic correction whenever \eqref{eq:gamma-flow-comparison} holds. The learned-score contribution in the recoverable range then decreases as well.
\section{Experiments}
We design three experiments to test our theoretical claims. We first use a diagonal Gaussian inverse problem with a closed-form posterior and examine whether the predicted transition across Hilbert-scale norms is visible in finite-dimensional simulations. We then use non-Gaussian one-dimensional signals and a learned score surrogate to test the same transition under an implicit prior representation. We finally compare isotropic and geometry-aware posterior guidance after matching the observation residual, so differences in reconstruction quality do not come from unequal data fit. For a spectral vector \(u=\sum_{j=1}^d u_j e_j\), all Hilbert-scale errors are computed as
\[
\|u\|_{H^s_d}^2=\sum_{j=1}^d j^{2s}|u_j|^2.
\]
Unless stated otherwise, reported values are empirical means over independent test draws.\footnote{All experimental results and codes can be found at \href{https://bmoraffa.github.io}{GitHub}.}

\subsection{Experiment 1: Known-Posterior Recoverability Transition}

We use the diagonal Gaussian inverse problem
\[
x_j\sim N(0,\tau_j^2),
\quad
\tau_j^2\propto j^{-(2\beta+1)},
\quad
(\mathcal A x)_j=a_jx_j,
\qquad
a_j=j^{-r},
\quad
y_j=a_jx_j^\dagger+\sigma\xi_j,
\quad
\xi_j\sim N(0,1).
\]
The parameter grid, Monte Carlo budget, learned-surrogate sample size, and weak perturbation level are listed in Table~\ref{tab:exp1-impl-neurips}. We evaluate the reconstruction error on the dense regularity grid
\[
s\in\{0,0.05,0.10,\ldots,2.50\}.
\]

The diagonal Gaussian structure gives the closed-form posterior
\[
x_j\mid y_j\sim N(m_j,v_j),
\qquad
v_j=\left(\tau_j^{-2}+\sigma^{-2}a_j^2\right)^{-1},
\qquad
m_j=v_j\sigma^{-2}a_jy_j .
\]
We compare four reconstructions, \(\widehat{x}_{\mathrm{post}},\) \(\widehat{x}_{\mathrm{oracle}},\) \(\widehat{x}_{\theta},\) and \(\widehat{x}_{\theta,\mathrm{weak}}\).
The reconstruction \(\widehat{x}_{\mathrm{post}}\) is the exact posterior mean, \(\widehat{x}_{\mathrm{oracle}}\) is an oracle posterior-diffusion surrogate formed by adding posterior-scale fluctuations to the exact mean, \(\widehat{x}_{\theta}\) uses an empirical prior variance estimated from \(600\) synthetic training draws, and \(\widehat{x}_{\theta,\mathrm{weak}}\) adds a weak-norm score perturbation with \(q=1\) and \(\delta=4\times 10^{-2}.\)
The perturbation is amplified by \(\max(a_j,5\times 10^{-3})^{-1}\), so that the stress concentrates in the weakly observed spectral tail. The main diagnostic is
\[
\mathcal{E}(s;r,\beta,\sigma)
=
\mathbb E\|\widehat{x}-x^\dagger\|_{H^s_d}^2,
\]
as well as the normalized curve \(\mathcal{E}(s)/\mathcal{E}(0)\) and the empirical recoverability index
\[
\widehat{s}_{\mathrm{crit}}
=
\sup\left\{
s:
\frac{\mathcal{E}(s;r,\beta,\sigma)}
{\mathcal{E}(0;r,\beta,\sigma)}
\le C_{\mathrm{tol}}
\right\},
\qquad
C_{\mathrm{tol}}\in\{5,10\}.
\]

\begin{table}[!htbp]
\centering
\caption{Implementation details for Experiment 1.}
\label{tab:exp1-impl-neurips}
\begin{tabular}{ll}
\toprule
Component & Setting \\
\midrule
Dimensions & \(d \in \{256,512,1024\}\) \\
Forward smoothing & \(r \in \{0.5,1,2,3\}\) \\
Truth smoothness & \(\beta \in \{1,2,3\}\) \\
Noise levels & \(\sigma \in \{10^{-3},10^{-2},3\times 10^{-2}\}\) \\
Target norms & \(s = 0,0.05,\dots,2.50\) \\
Trials per configuration & \(120\) \\
Learned surrogate training draws & \(600\) \\
Weak perturbation & \((q,\delta)=(1,4\times 10^{-2})\) \\
\bottomrule
\end{tabular}
\end{table}

Figure~\ref{fig:exp1-transition} shows the normalized error curves at \(d=1024\) and \(\sigma=10^{-2}\). Across all values of \(r\), the curves are flat for small \(s\) and increase rapidly once the norm gives enough weight to unresolved high-frequency coordinates. This transition is the main finite-dimensional signature of the recoverability boundary. Figure~\ref{fig:exp1-heatmaps} summarizes the same behavior through \(\widehat{s}_{\mathrm{crit}}\). The recoverability index increases with the truth smoothness \(\beta\), which agrees with the role of the source condition in the theory.

\begin{figure}[!htbp]
\centering
\includegraphics[width=\linewidth]{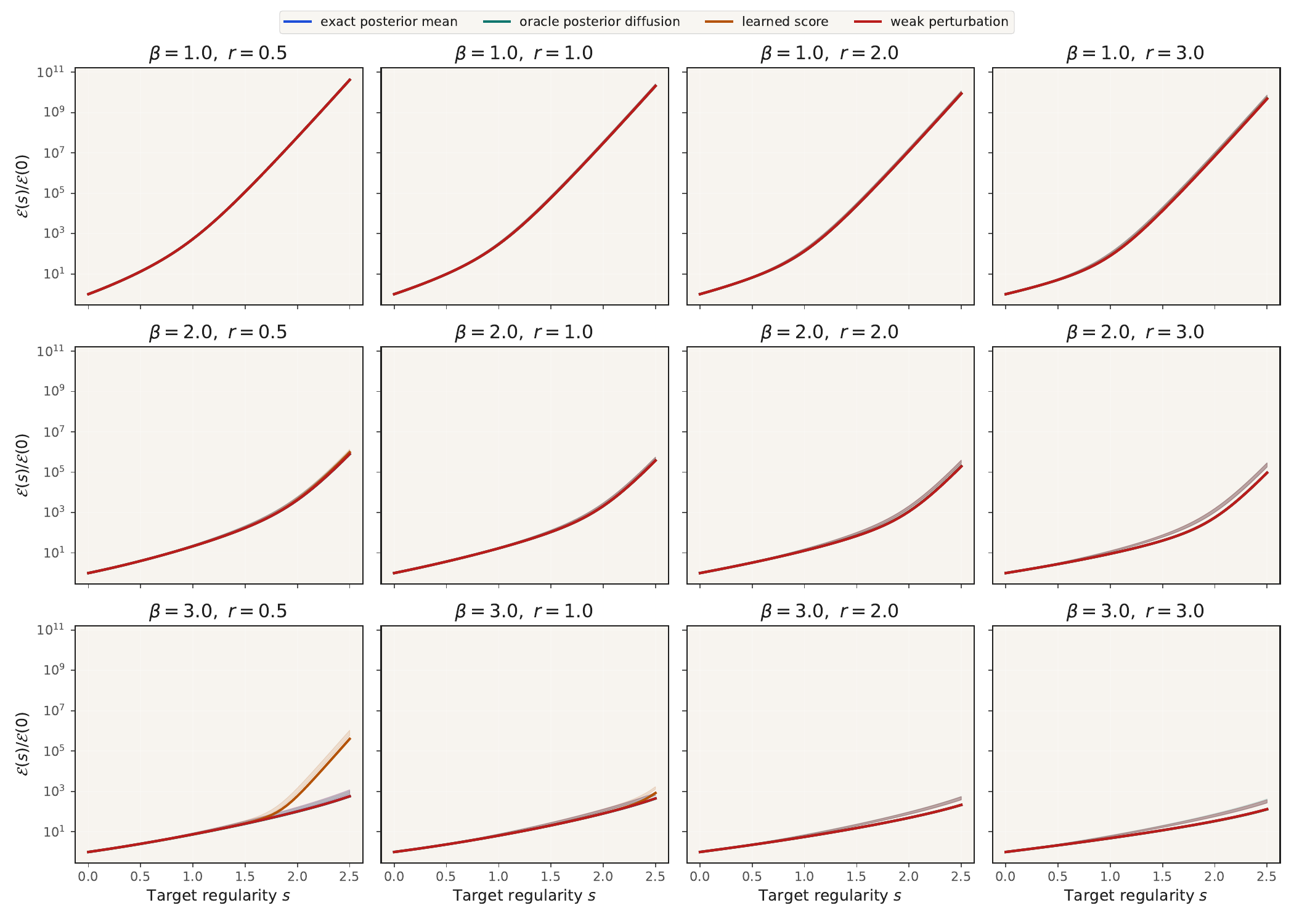}
\caption{Dense transition atlas. Each panel tracks normalized error across target regularity using a denser $s$ grid and confidence ribbons.}
\label{fig:exp1-transition}
\end{figure}

\begin{figure}[!htbp]
\centering
\includegraphics[width=\linewidth]{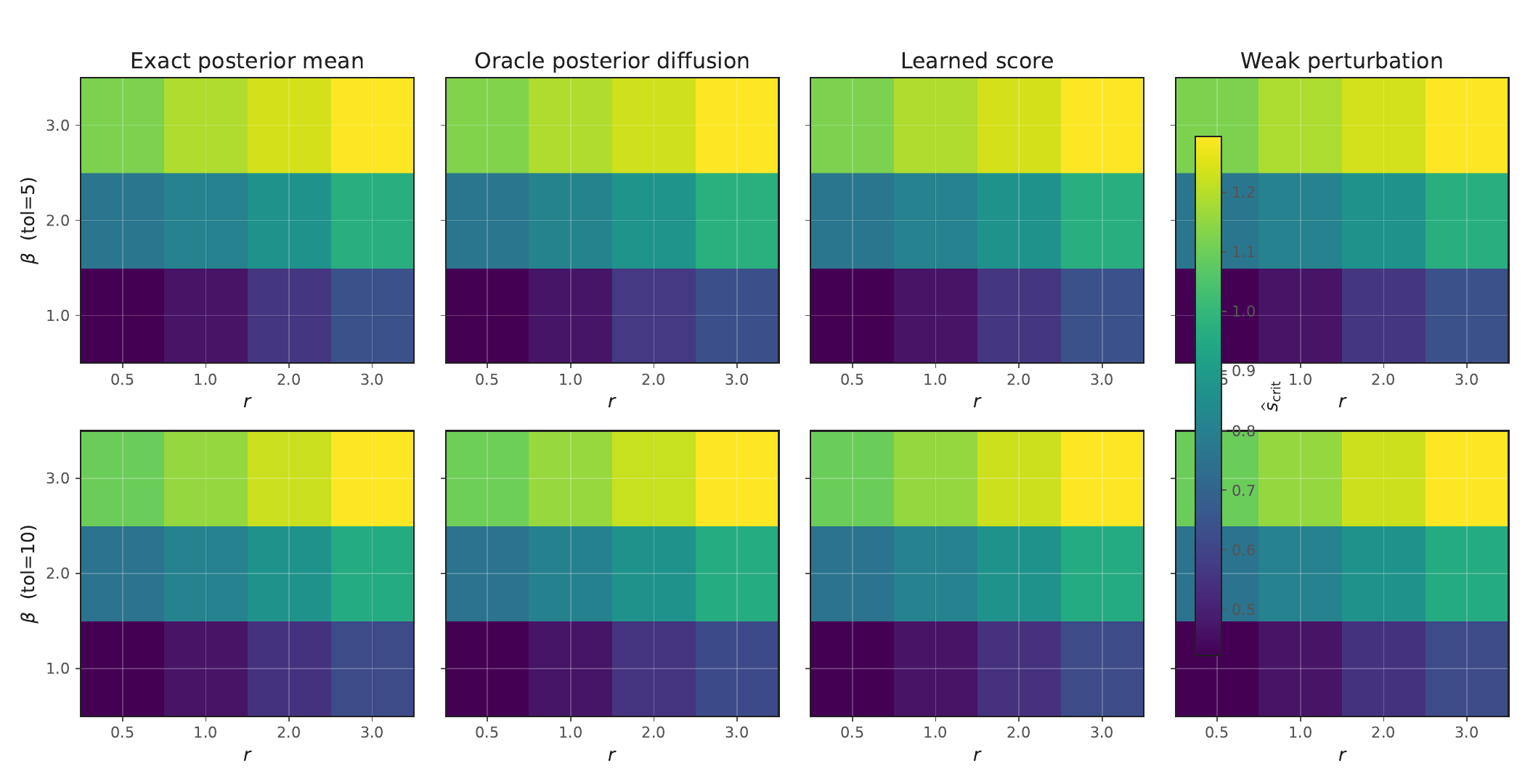}
\caption{Empirical recoverability heatmaps under two tolerance levels. The panels summarize how the critical regularity varies with $(r,\beta)$.}
\label{fig:exp1-heatmaps}
\end{figure}

Table~\ref{tab:exp1-scrit} simulates \(\widehat{s}_{\mathrm{crit}}\) at \(d=1024\), \(\sigma=10^{-2}\), and \(C_{\mathrm{tol}}=5\). For the exact posterior mean, \(\widehat{s}_{\mathrm{crit}}\) lies between \(0.30\) and \(0.46\) for \(\beta=1\), between \(0.55\) and \(0.70\) for \(\beta=2\), and between \(0.81\) and \(0.93\) for \(\beta=3\). The learned and weakly perturbed variants give nearly identical critical indices, which shows that this perturbation level changes the estimator much less than the change induced by the target norm \(H^s_d\).

\begin{table}[!htbp]
\centering
\small
\caption{Empirical recoverability index $\widehat s_{\mathrm{crit}}$ at $d=1024$, $\sigma=10^{-2}$, and tolerance $C_{\mathrm{tol}}=5$. Larger values indicate stable recovery deeper into stronger Hilbert norms.}
\label{tab:exp1-scrit}
\resizebox{\linewidth}{!}{%
\begin{tabular}{lccccccccc}
\toprule
$r$ & $\beta=1$ post & $\beta=1$ learned & $\beta=1$ weak & $\beta=2$ post & $\beta=2$ learned & $\beta=2$ weak & $\beta=3$ post & $\beta=3$ learned & $\beta=3$ weak \\
\midrule
$0.5$ & 0.30 & 0.30 & 0.30 & 0.55 & 0.55 & 0.55 & 0.81 & 0.81 & 0.81 \\
$1.0$ & 0.34 & 0.34 & 0.34 & 0.58 & 0.58 & 0.58 & 0.85 & 0.85 & 0.85 \\
$2.0$ & 0.40 & 0.40 & 0.40 & 0.62 & 0.62 & 0.62 & 0.89 & 0.89 & 0.89 \\
$3.0$ & 0.46 & 0.46 & 0.46 & 0.70 & 0.70 & 0.70 & 0.93 & 0.93 & 0.93 \\
\bottomrule
\end{tabular}
}
\end{table}

Table~\ref{tab:exp1-level-errors} gives the error scale at \(s=0,1,2\) for the representative case \(d=1024\), \(\beta=2\), and \(\sigma=10^{-2}\). For the exact posterior mean,
\[
\mathcal{E}(0)=1.242\times 10^{-3},
\qquad
\mathcal{E}(1)=2.657\times 10^{-2},
\qquad
\mathcal{E}(2)=5.203
\]
when \(r=0.5\), and
\[
\mathcal{E}(0)=1.056\times 10^{-2},
\qquad
\mathcal{E}(1)=9.586\times 10^{-2},
\qquad
\mathcal{E}(2)=5.958
\]
when \(r=3.0\). These values show the strong growth of error as the reconstruction norm moves from \(H^0_d\) to \(H^2_d\).
Figure~\ref{fig:exp1-stress} and Table~\ref{tab:exp1-stress-ratio} demonstrate the effect of the weak score perturbation. For the severe case \(r=3\) and \(\beta=2\), the ratio
\[
\frac{\mathcal{E}_{\theta,\mathrm{weak}}(2)}
{\mathcal{E}_{\mathrm{post}}(2)}
\]
stays between approximately \(0.998\) and \(1.000\) across the reported \((d,\sigma)\) grid. Thus, the experiment clearly resolves the transition in \(s\), although the present perturbation is too mild to produce a large departure from the exact Bayesian benchmark.

\begin{table}[!htbp]
\centering
\caption{Mean reconstruction error at three target regularities for the representative case $d=1024$, $\beta=2$, and $\sigma=10^{-2}$.}
\label{tab:exp1-level-errors}
\begin{tabular}{llccc}
\toprule
$r$ & Method & $s=0$ & $s=1$ & $s=2$ \\
\midrule
$0.5$ & post & 1.242e-03 & 2.657e-02 & 5.203e+00 \\
$0.5$ & oracle & 1.291e-03 & 2.752e-02 & 5.362e+00 \\
$0.5$ & theta & 1.243e-03 & 2.660e-02 & 5.476e+00 \\
$0.5$ & theta\_weak & 1.242e-03 & 2.657e-02 & 5.203e+00 \\
$1.0$ & post & 2.600e-03 & 4.388e-02 & 5.452e+00 \\
$1.0$ & oracle & 2.698e-03 & 4.550e-02 & 5.628e+00 \\
$1.0$ & theta & 2.601e-03 & 4.390e-02 & 5.452e+00 \\
$1.0$ & theta\_weak & 2.599e-03 & 4.387e-02 & 5.452e+00 \\
$2.0$ & post & 5.058e-03 & 6.489e-02 & 5.692e+00 \\
$2.0$ & oracle & 5.049e-03 & 6.532e-02 & 5.860e+00 \\
$2.0$ & theta & 5.063e-03 & 6.493e-02 & 5.693e+00 \\
$2.0$ & theta\_weak & 5.060e-03 & 6.492e-02 & 5.694e+00 \\
$3.0$ & post & 1.056e-02 & 9.586e-02 & 5.958e+00 \\
$3.0$ & oracle & 1.084e-02 & 9.816e-02 & 6.124e+00 \\
$3.0$ & theta & 1.056e-02 & 9.588e-02 & 5.958e+00 \\
$3.0$ & theta\_weak & 1.051e-02 & 9.539e-02 & 5.955e+00 \\
\bottomrule
\end{tabular}
\end{table}

\begin{figure}[!htbp]
\centering
\includegraphics[width=\linewidth]{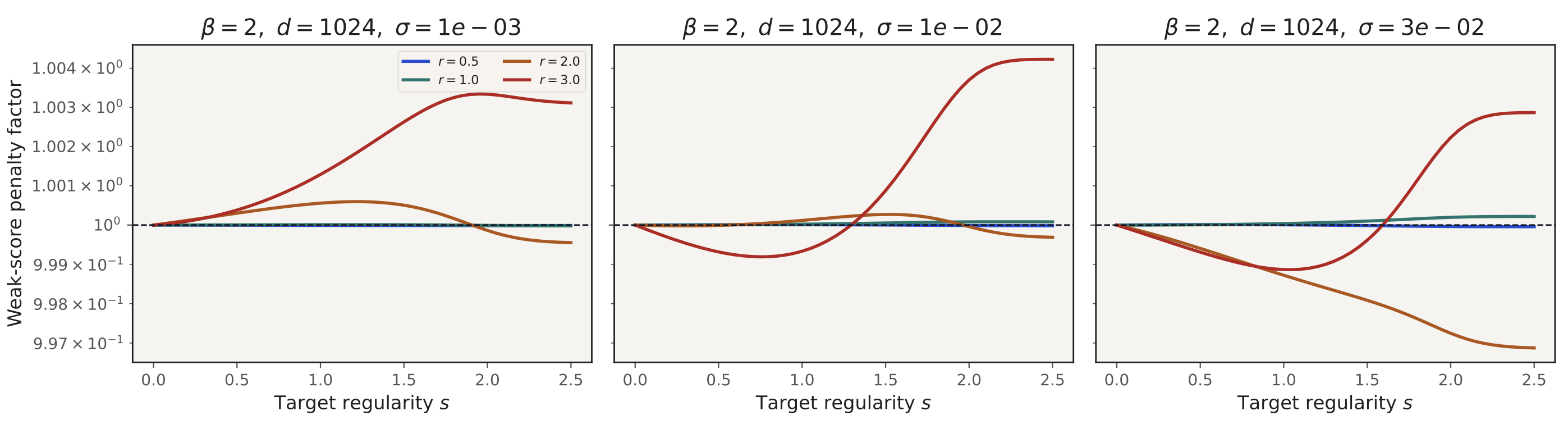}
\caption{Weak-score stress plot. Values above one indicate amplification of the strong-norm error relative to the exact posterior mean.}
\label{fig:exp1-stress}
\end{figure}

\begin{table}[!htbp]
\centering
\caption{Strong-norm stress ratio for the weakly perturbed score under the more ill-posed case $r=3$ and $\beta=2$.}
\label{tab:exp1-stress-ratio}
\begin{tabular}{ccc}
\toprule
$d$ & $\sigma$ & $\mathcal{E}_{\theta,\mathrm{weak}}(2)/\mathcal{E}_{\mathrm{post}}(2)$ \\
\midrule
$256$ & $1e-03$ & 0.998 \\
$256$ & $1e-02$ & 1.000 \\
$256$ & $3e-02$ & 1.000 \\
$512$ & $1e-03$ & 1.000 \\
$512$ & $1e-02$ & 0.999 \\
$512$ & $3e-02$ & 0.999 \\
$1024$ & $1e-03$ & 1.000 \\
$1024$ & $1e-02$ & 0.999 \\
$1024$ & $3e-02$ & 0.999 \\
\bottomrule
\end{tabular}
\end{table}

\FloatBarrier
\subsection{Experiment 2: Learned-Score Invisible-Mode Stress Test}

We generate non-Gaussian one-dimensional signals on a grid with \(d=256\). Each signal combines Gaussian pulses, a small number of sinusoidal components, and a weak rough tail in the discrete cosine basis, after which the samples are centered and rescaled to unit empirical standard deviation. We use three diagonal forward operators,
\[
\mathcal A_1=\text{mild blur},
\qquad
\mathcal A_2=\text{severe blur},
\qquad
\mathcal A_3=\text{severe blur with missing frequencies},
\]
with spectral multipliers
\[
a^{(1)}_j=\frac{1}{1+j^{1.1}},
\qquad
a^{(2)}_j=\frac{1}{1+j^{2.2}},
\]
where \(\mathcal A_3\) also includes hard spectral masks. The observation noise is fixed at \(\sigma=0.04\). We use a polynomial denoising-score surrogate as a lightweight learned prior. For each diffusion-noise level \(\alpha\), the score is regressed coordinatewise onto the odd polynomial features \(x, x^3, x^5\)
with ridge regularization. 
The training, validation, and test splits contain \(1800\), \(320\), and \(54\) signals, respectively. The reverse sampler runs \(8\) chains of annealed Langevin dynamics over \(10\) geometrically spaced noise levels from \(1.0\) to \(0.03\) and takes \(4\) steps per level. Table~\ref{tab:exp2-impl-neurips} shows the selected Tikhonov and Sobolev parameters, the guidance strengths of DPS and GeoPD, and the  implementation details. We compare the Tikhonov, Sobolev, DPS, and GeoPD reconstructions \(\widehat{x}_{\mathrm{Tik}}\),
\(\widehat{x}_{\mathrm{Sob}}\), \(\widehat{x}_{\mathrm{DPS}}\), and\(\widehat{x}_{\mathrm{GeoPD}}\). The evaluation measures the
reconstruction error in the \(H^0_d\), \(H^1_d\), and \(H^2_d\) norms, as well as the error in the low-, mid-, and high-frequency bands.

\begin{table}[!htbp]
\centering
\caption{Implementation details for Experiment 2.}
\label{tab:exp2-impl-neurips}
\begin{tabular}{ll}
\toprule
Component & Setting \\
\midrule
Signal dimension & \(d=256\) \\
Noise level & \(\sigma=0.04\) \\
Train / validation / test & \(1800 / 320 / 54\) \\
Score features & \(x\), \(x^3\), \(x^5\) \\
Noise schedule & \(10\) levels from \(1.0\) to \(0.03\) \\
Reverse sampler & \(4\) steps per level, \(8\) chains \\
Tikhonov parameter & \(\lambda_{\mathrm{Tik}}=4.959\times 10^{-4}\) \\
Sobolev parameter & \(\lambda_{\mathrm{Sob}}=10^{-4}\) \\
Guidance strengths & DPS \((1.7,0.8,0.8)\), GeoPD \((0.8,0.8,0.8)\) \\
\bottomrule
\end{tabular}
\end{table}

Figure~\ref{fig:exp2-main} shows a representative reconstruction under \(\mathcal A_3\), the spectral energy and error profile, and the high-frequency error across the three forward operators. The reconstruction plot illustrates the missing-frequency stress test, the spectral panel shows where the diffusion reconstructions differ in the tail, and the high-frequency panel isolates the error most relevant to the Hilbert-scale theory. Table~\ref{tab:exp2-metrics} shows the corresponding \(H^0_d\), \(H^1_d\), \(H^2_d\), and high-frequency errors.

\begin{figure}[!htbp]
\centering
\includegraphics[width=\linewidth]{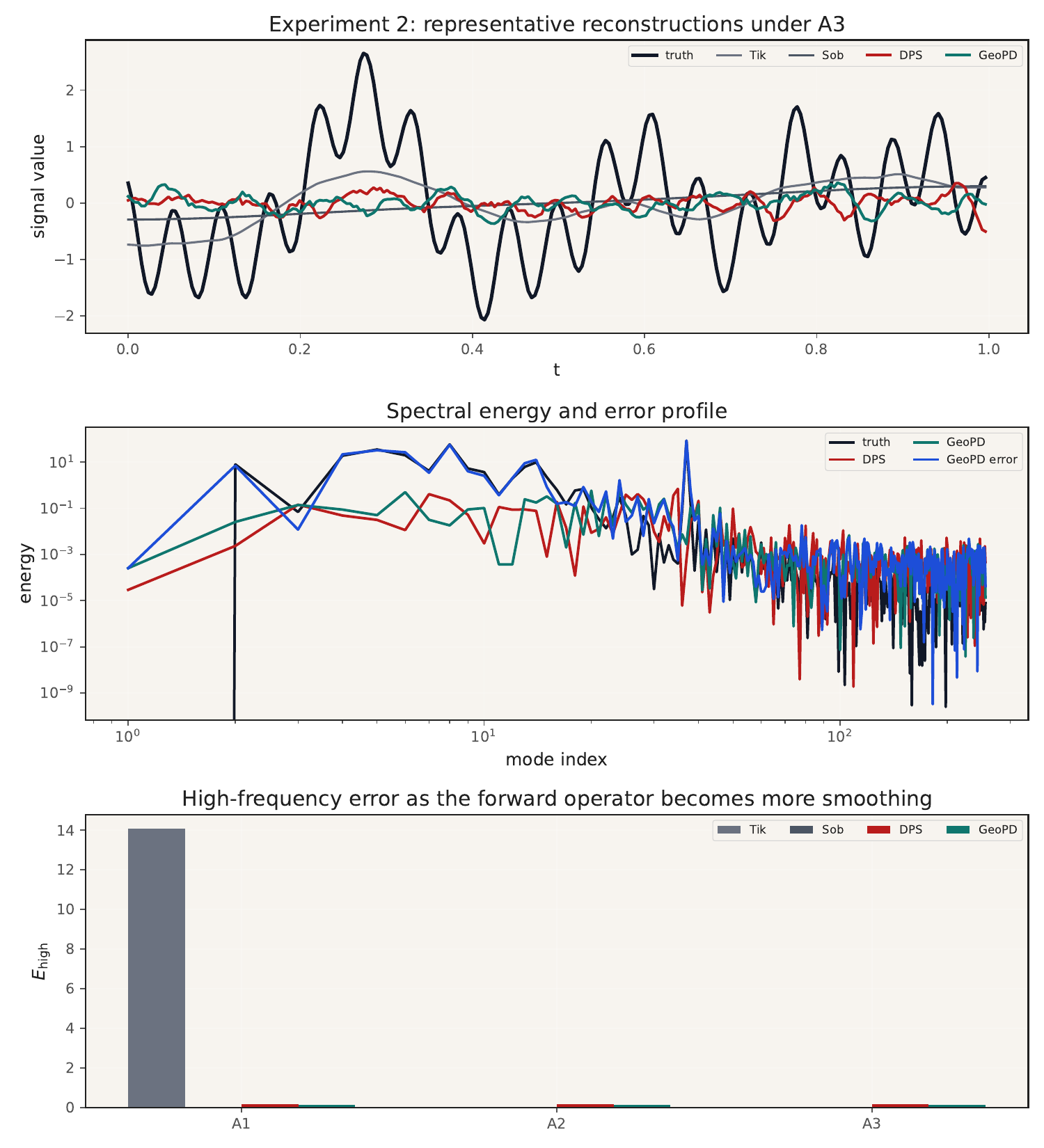}
\caption{Top: representative reconstructions for the missing-frequency operator. Middle: spectral energy of the truth, reconstructions, and reconstruction error. Bottom: high-frequency band error across forward operators.}
\label{fig:exp2-main}
\end{figure}

\begin{table}[!htbp]
\centering
\small
\caption{Reconstruction metrics for non-Gaussian one-dimensional signals. The geometry-aware sampler is most distinct from isotropic guidance in stronger norms and in the high-frequency band.}
\label{tab:exp2-metrics}
\begin{tabular}{llcccc}
\toprule
Operator & Method & $H^0$ & $H^1$ & $H^2$ & $E_{\mathrm{high}}$ \\
\midrule
A1 & Tik & 8.010e+01 & 4.802e+05 & 1.170e+10 & 1.408e+01 \\
A1 & Sob & 1.849e+02 & 8.573e+04 & 9.019e+07 & 1.802e-02 \\
A1 & DPS & 2.553e+02 & 9.791e+04 & 2.497e+08 & 1.556e-01 \\
A1 & GeoPD & 2.676e+02 & 9.814e+04 & 2.395e+08 & 1.215e-01 \\
A2 & Tik & 1.597e+02 & 8.466e+04 & 9.028e+07 & 1.838e-02 \\
A2 & Sob & 2.158e+02 & 8.640e+04 & 9.021e+07 & 1.802e-02 \\
A2 & DPS & 2.661e+02 & 9.819e+04 & 2.498e+08 & 1.556e-01 \\
A2 & GeoPD & 2.680e+02 & 9.814e+04 & 2.395e+08 & 1.215e-01 \\
A3 & Tik & 1.597e+02 & 8.466e+04 & 9.020e+07 & 1.827e-02 \\
A3 & Sob & 2.158e+02 & 8.640e+04 & 9.021e+07 & 1.802e-02 \\
A3 & DPS & 2.661e+02 & 9.819e+04 & 2.498e+08 & 1.556e-01 \\
A3 & GeoPD & 2.680e+02 & 9.814e+04 & 2.395e+08 & 1.215e-01 \\
\bottomrule
\end{tabular}
\end{table}

We also evaluate validation score mismatch in weak norms. At the representative diffusion level \(\alpha\approx 9.65\times 10^{-2}\), the estimated weak-norm score errors are
\[
\widehat{\delta}_0^2(\alpha)=5.767\times 10^3,
\qquad
\widehat{\delta}_{0.5}^2(\alpha)=3.860\times 10^2,
\qquad
\widehat{\delta}_{1}^2(\alpha)=6.946\times 10^1,
\qquad
\widehat{\delta}_{2}^2(\alpha)=9.216.
\]
At the smallest noise level \(\alpha=0.03\), the four values grow to
\(9.480\times 10^4\), \(4.388\times 10^3\), \(7.188\times 10^2\), and
\(9.317\times 10^1\). The decrease from each norm to the next weaker norm matches the theory. A more negative norm index places less weight on high frequencies, where the score mismatch concentrates.

The Sobolev regularizer achieves the lowest reconstruction error for each of the three operators. For $\mathcal A_1$, the squared error $H^2_d$  of the Sobolev regularizer is $9.019\times 10^7$, compared to $2.497\times 10^8$ for DPS and $2.395\times 10^8$ for GeoPD (Table~\ref{tab:exp2-metrics}). For $\mathcal A_2$ and $\mathcal A_3$, DPS also has the largest squared $H^2_d$ error. GeoPD outperforms DPS for every operator by approximately $4\%$ in squared $H^2_d$ error. The improvement is substantially larger in the high-frequency band, where GeoPD reduces the band error from approximately $1.556\times 10^{-1}$ to $1.215\times 10^{-1}$ relative to DPS. The band error therefore demosntrates the effect of the geometry-aware posterior correction that the Hilbert-scale analysis predicts.
\FloatBarrier
\subsection{Experiment 3: Matched-Residual Comparison of Isotropic and Geometry-Aware Guidance}

We compare isotropic and geometry-aware posterior guidance after matching the observation residual. Both samplers use the same score model, reverse-time integrator, and terminal distribution; they differ only in the likelihood correction. The isotropic sampler uses a correction proportional to \(\mathcal A^\ast(y-\mathcal A x)\), and the geometry-aware sampler uses
\[
\Gamma_t=(\mathcal A^\ast \mathcal A+\lambda_t L^\rho)^{-1}.
\]
In diagonal spectral coordinates,
\[
(\Gamma_t z)_j=\frac{1}{a_j^2+\lambda_t j^{2\rho}}z_j,
\qquad
\rho=2.
\]

We evaluate this comparison in two domains. The spectral domain uses
\[
d=512,\qquad \beta=2,\qquad \sigma=10^{-2},\qquad r\in\{0.5,1,2,3\}.
\]
The learned-signal domain reuses the score surrogate from Experiment~2 and focuses on the severe operators \(\mathcal A_2\) and \(\mathcal A_3\). Guidance strengths are selected from
\[
\gamma\in\{0.25,0.40,0.55,0.70,0.85,1.00,1.20,1.45,1.70,2.00,2.30\}.
\]
For each scenario, we sweep this grid separately for isotropic and geometry-aware guidance, record the mean residual, and choose the pair \((\gamma_{\mathrm{iso}},\gamma_{\mathrm{geo}})\) with the smallest residual mismatch. We then compare the reconstruction errors in \(H^0_d\), \(H^1_d\), and \(H^2_d\), along with coefficient-wise error and posterior sample variability.

\begin{table}[!htbp]
\centering
\caption{Implementation details for Experiment 3.}
\label{tab:exp3-impl-neurips}
\begin{tabular}{ll}
\toprule
Component & Setting \\
\midrule
Spectral domain & \(d=512\), \(\beta=2\), \(\sigma=10^{-2}\), \(r\in\{0.5,1,2,3\}\) \\
Learned-signal domain & \(d=256\), operators \(\mathcal A_2\) and \(\mathcal A_3\) \\
Guidance sweep & \(\gamma \in \{0.25,0.40,0.55,0.70,0.85,1.00,1.20,1.45,1.70,2.00,2.30\}\) \\
Sweep norms & \(s\in\{0,0.5,1,1.5,2\}\) \\
Reported matched norms & \(H^0_d\), \(H^1_d\), \(H^2_d\) \\
Reverse sampler & \(9\) levels, \(4\) steps per level, \(10\) chains \\
\bottomrule
\end{tabular}
\end{table}

Table~\ref{tab:exp3-gamma} summarizes the selected guidance strengths and residual gaps. In the spectral domain, the matched pairs are
\[
(\gamma_{\mathrm{iso}},\gamma_{\mathrm{geo}})
=
(1.70,2.30),\ (1.70,2.30),\ (1.00,2.00),\ (0.40,0.70)
\]
for \(r=0.5,1.0,2.0,3.0\), with residual gaps between \(3.396\times 10^{-4}\) and \(3.969\times 10^{-3}\). In the learned-signal domain, the matched pairs are \((0.85,0.85)\) for \(\mathcal A_2\) and \((1.00,0.40)\) for \(\mathcal A_3\). Figure~\ref{fig:exp3-stability} shows the residual-error tradeoff curves before matching, which verifies that the comparison is made at comparable data fit.

\begin{table}[!htbp]
\centering
\caption{Guidance strengths selected for the matched-residual comparison}
\label{tab:exp3-gamma}
\begin{tabular}{llccc}
\toprule
Domain & Scenario & $\gamma_{\mathrm{iso}}$ & $\gamma_{\mathrm{geo}}$ & Residual gap \\
\midrule
spectral & r=0.5 & 1.70 & 2.30 & 3.396e-04 \\
spectral & r=1.0 & 1.70 & 2.30 & 3.969e-03 \\
spectral & r=2.0 & 1.00 & 2.00 & 7.609e-04 \\
spectral & r=3.0 & 0.40 & 0.70 & 2.247e-03 \\
learned\_signal & A2 & 0.85 & 0.85 & 1.558e-05 \\
learned\_signal & A3 & 1.00 & 0.40 & 6.957e-04 \\
\bottomrule
\end{tabular}
\end{table}

\begin{figure}[!htbp]
\centering
\includegraphics[width=\linewidth]{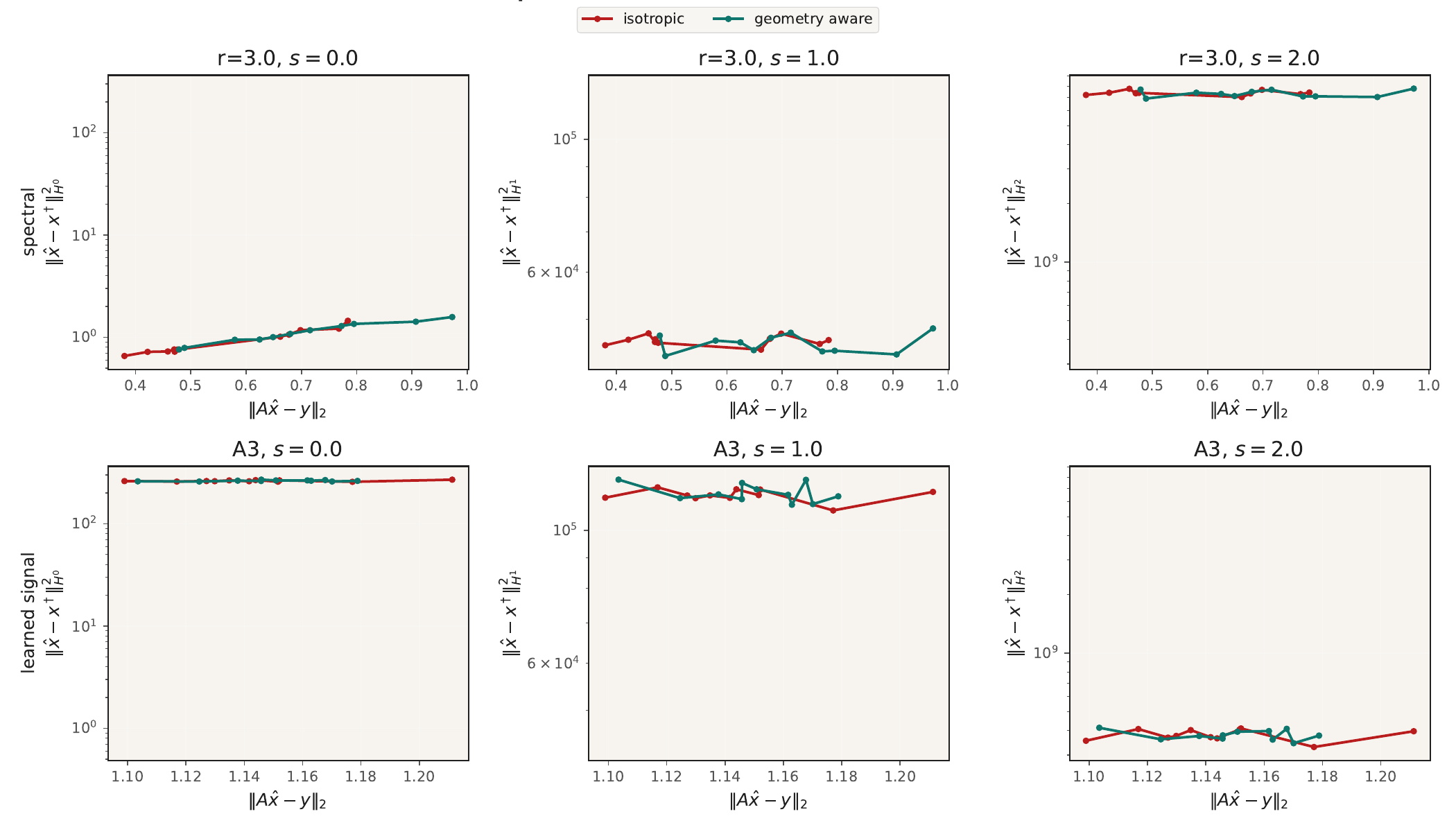}
\caption{Residual-error tradeoff curves for representative severe cases in Experiment 3. Each curve traces the effect of guidance strength before the matched-residual point is selected.}
\label{fig:exp3-stability}
\end{figure}

Table~\ref{tab:exp3-metrics} gives the matched-residual reconstruction errors. In the spectral domain, geometry-aware guidance gives small gains in most cases. For the severe case \(r=3.0\), the matched residuals are
\[
\|\mathcal A\widehat{x}_{\mathrm{iso}}-y\|_2 \approx 7.859\times 10^{-1},
\qquad
\|\mathcal A\widehat{x}_{\mathrm{geo}}-y\|_2 \approx 7.827\times 10^{-1},
\]
and the corresponding errors are
\[
H^0_d:\ 1.510 \text{ vs. } 1.446,
\qquad
H^1_d:\ 4.685\times 10^4 \text{ vs. } 4.637\times 10^4,
\]
\[
H^2_d:\ 7.489\times 10^9 \text{ vs. } 7.392\times 10^9.
\]
Table~\ref{tab:exp3-gain} and Figure~\ref{fig:exp3-gain} summarize these ratios. For \(r=3.0\), the gain ratios are \(1.044\) in \(H^0_d\), \(1.010\) in \(H^1_d\), and \(1.013\) in \(H^2_d\). For \(r=2.0\), the largest gain occurs in \(H^0_d\), with ratio approximately \(1.149\).

\begin{table}[!htbp]
\centering
\small
\caption{Experiment 3 matched-residual comparison of isotropic and geometry-aware posterior guidance. Residuals are tuned to be comparable within each scenario before comparing reconstruction stability.}
\label{tab:exp3-metrics}
\resizebox{\linewidth}{!}{%
\begin{tabular}{lllcccc}
\toprule
Domain & Scenario & Method & Residual & $H^0$ & $H^1$ & $H^2$ \\
\midrule
spectral & r=0.5 & isotropic & 5.297e-01 & 7.801e-01 & 4.592e+04 & 7.389e+09 \\
spectral & r=0.5 & geometry\_aware & 5.447e-01 & 8.454e-01 & 4.572e+04 & 7.317e+09 \\
spectral & r=1.0 & isotropic & 4.774e-01 & 7.793e-01 & 4.662e+04 & 7.528e+09 \\
spectral & r=1.0 & geometry\_aware & 4.449e-01 & 7.328e-01 & 4.676e+04 & 7.550e+09 \\
spectral & r=2.0 & isotropic & 5.381e-01 & 8.918e-01 & 4.218e+04 & 6.670e+09 \\
spectral & r=2.0 & geometry\_aware & 4.721e-01 & 7.765e-01 & 4.224e+04 & 6.664e+09 \\
spectral & r=3.0 & isotropic & 7.859e-01 & 1.510e+00 & 4.685e+04 & 7.489e+09 \\
spectral & r=3.0 & geometry\_aware & 7.827e-01 & 1.446e+00 & 4.637e+04 & 7.392e+09 \\
learned\_signal & A2 & isotropic & 1.167e+00 & 2.585e+02 & 1.030e+05 & 3.550e+08 \\
learned\_signal & A2 & geometry\_aware & 1.186e+00 & 2.671e+02 & 1.071e+05 & 3.613e+08 \\
learned\_signal & A3 & isotropic & 1.018e+00 & 2.590e+02 & 1.209e+05 & 3.639e+08 \\
learned\_signal & A3 & geometry\_aware & 1.045e+00 & 2.616e+02 & 1.228e+05 & 3.907e+08 \\
\bottomrule
\end{tabular}
}
\end{table}

\begin{table}[!htbp]
\centering
\caption{Relative gain ratio $\|\hat x_{\mathrm{iso}}-x^\dagger\|_{H^s}^2 / \|\hat x_{\mathrm{geo}}-x^\dagger\|_{H^s}^2$ at matched residual in Experiment 3.}
\label{tab:exp3-gain}
\begin{tabular}{llcc}
\toprule
Domain & Scenario & $s$ & Gain ratio \\
\midrule
spectral & r=0.5 & $0$ & 0.923 \\
spectral & r=0.5 & $1$ & 1.004 \\
spectral & r=0.5 & $2$ & 1.010 \\
spectral & r=1.0 & $0$ & 1.064 \\
spectral & r=1.0 & $1$ & 0.997 \\
spectral & r=1.0 & $2$ & 0.997 \\
spectral & r=2.0 & $0$ & 1.149 \\
spectral & r=2.0 & $1$ & 0.999 \\
spectral & r=2.0 & $2$ & 1.001 \\
spectral & r=3.0 & $0$ & 1.044 \\
spectral & r=3.0 & $1$ & 1.010 \\
spectral & r=3.0 & $2$ & 1.013 \\
learned\_signal & A2 & $0$ & 0.968 \\
learned\_signal & A2 & $1$ & 0.961 \\
learned\_signal & A2 & $2$ & 0.983 \\
learned\_signal & A3 & $0$ & 0.990 \\
learned\_signal & A3 & $1$ & 0.984 \\
learned\_signal & A3 & $2$ & 0.931 \\
\bottomrule
\end{tabular}
\end{table}

\begin{figure}[!htbp]
\centering
\includegraphics[width=\linewidth]{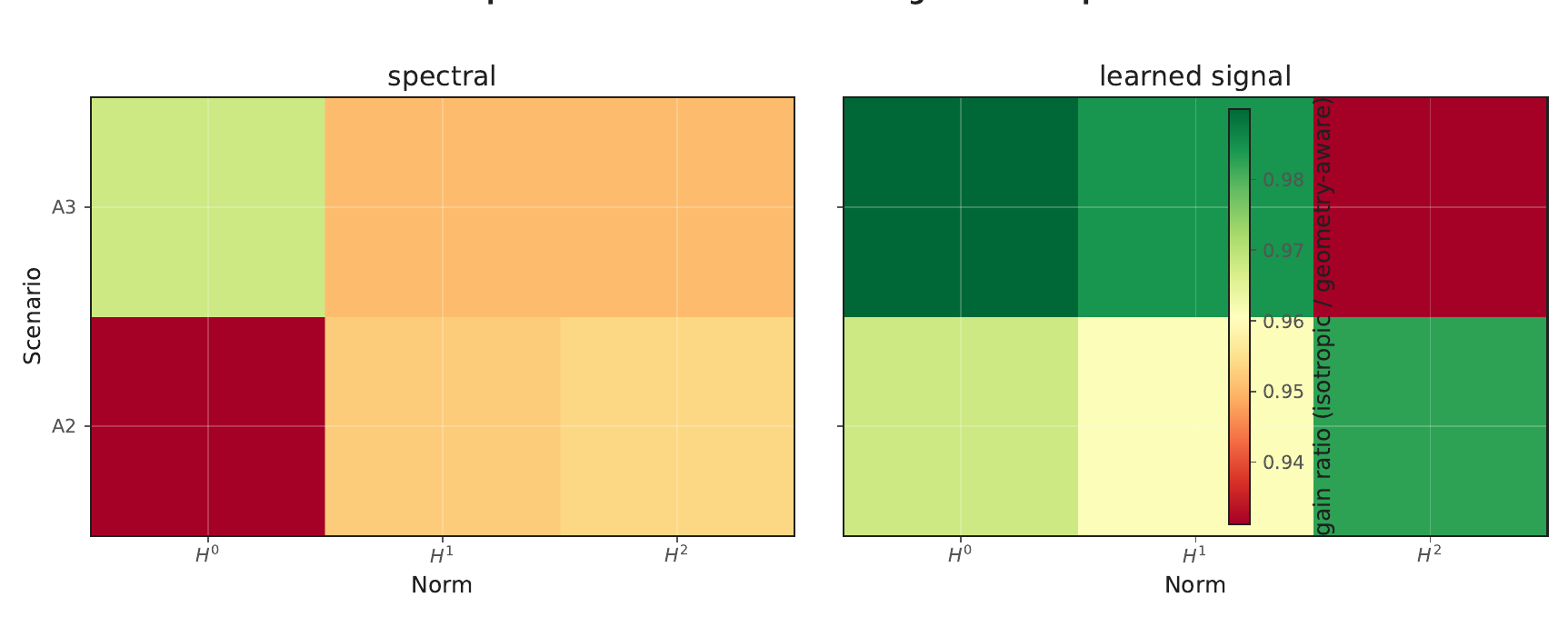}
\caption{Matched-residual gain heatmaps for Experiment 3. Values above one favor geometry-aware guidance.}
\label{fig:exp3-gain}
\end{figure}

The learned-signal domain gives a more mixed comparison. For \(\mathcal A_2\), isotropic guidance has slightly lower error in all three reported norms,
\[
H^0_d:\ 2.585\times 10^2 \text{ vs. } 2.671\times 10^2,
\]
\[
H^1_d:\ 1.030\times 10^5 \text{ vs. } 1.071\times 10^5,
\]
\[
H^2_d:\ 3.550\times 10^8 \text{ vs. } 3.613\times 10^8.
\]
The same pattern holds for \(\mathcal A_3\), where the \(H^2_d\) gain ratio is approximately \(0.931\). Figure~\ref{fig:exp3-modewise} gives additional spectral detail: in the severe spectral case, the geometry-aware correction reduces the high-mode tail of both coefficient-wise error and posterior sample variability, which agrees with the intended role of operator-adapted likelihood injection.

\begin{figure}[!htbp]
\centering
\includegraphics[width=\linewidth]{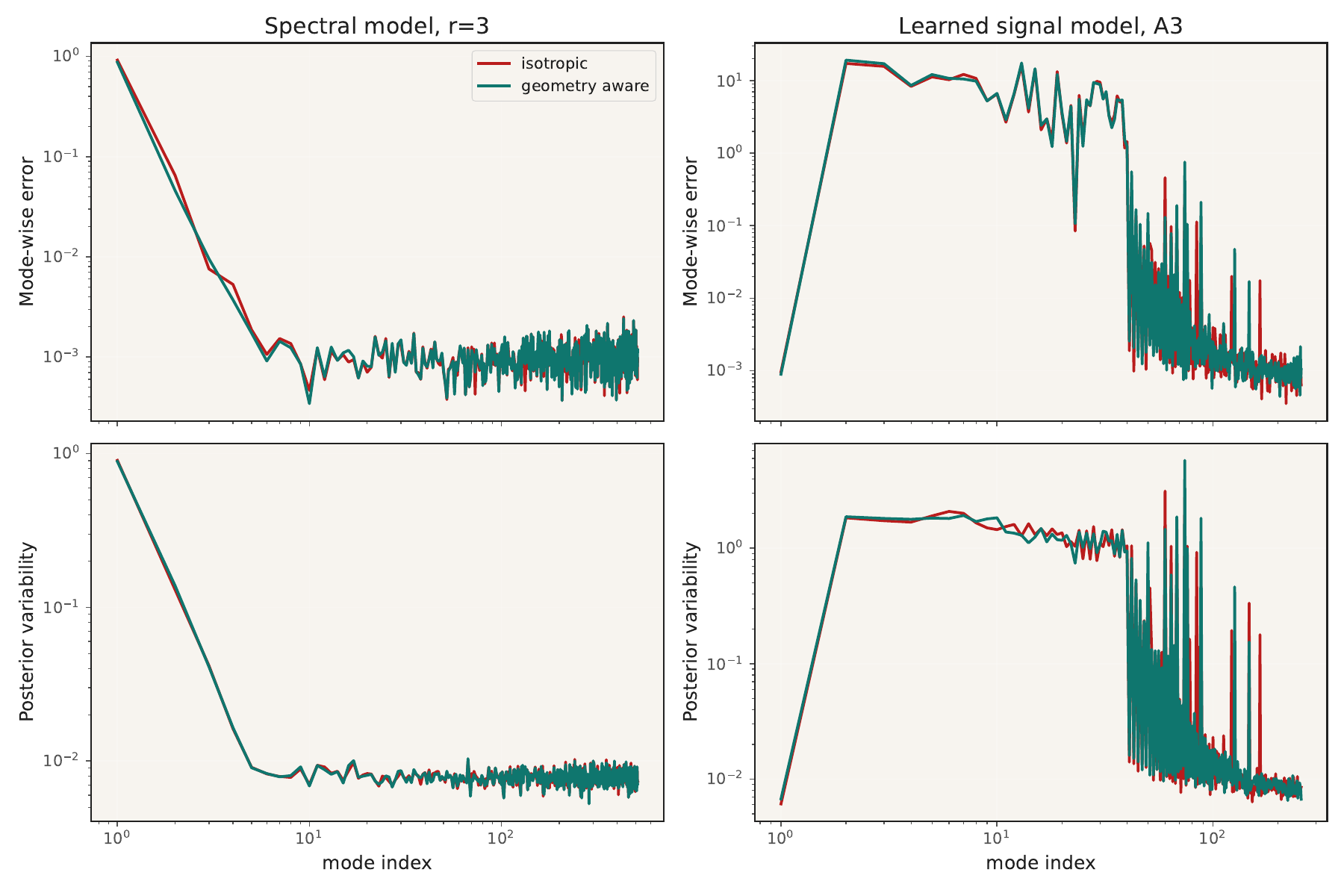}
\caption{Mode-wise reconstruction error and posterior variability for representative severe cases in Experiment 3.}
\label{fig:exp3-modewise}
\end{figure}

\FloatBarrier

\section{Conclusion}

We studied posterior diffusion for ill-posed linear inverse problems on Hilbert scales and showed that recovery depends on operator smoothing, truth regularity, weak-norm score error, the reconstruction norm, and the posterior correction. The recovery estimate separates observational noise, learned-score approximation, and reverse-time discretization, with the score contribution measured by a propagation functional from \(\Hs^{-q}\) to \(\Hs^s\). This functional identifies the recoverable Hilbert-scale norms in which weak-norm score error is stable after propagation through the posterior reverse dynamics. The comparison between isotropic and operator-adapted likelihood corrections shows that adapting the correction to the forward operator can reduce the score-propagation factor in the regularity ranges most affected by smoothing. The experiments support the theoretical prediction in the diagonal Gaussian model and show the same spectral-tail effects with learned score surrogates, although the gains depend on the score model and posterior-guidance calibration. These results indicate that diffusion priors for inverse problems should be evaluated through ambient reconstruction error and through the regularity scale in which recovery is measured, since Hilbert-scale error reveals when reconstruction depends on poorly observed high-frequency components.

\bibliography{references}

%%%%%%%%%%%%%%%%%%%%%%%%%%%%%%%%%%%%%%%%%%%%%%%%%%%%%%%%%%%%

\appendix
\section{Proof of Proposition~\ref{prop:wellposedness}}
\label{app-proof-wellposedness}

\begin{proof}
We prove this proposition for \(X_t^\star\). The proof for \(X_t^\theta\) is identical. Consider the space
\[
\mathbb S^2(\Hs^s)
=
\left\{
X \text{ adapted and continuous on }[t_0,1]
\; ; \;
\mathbb E\sup_{t\in[t_0,1]}\|X_t\|_{\Hs^s}^2<\infty
\right\},
\]
equipped with the norm
\[
\|X\|_{\mathbb S^2(\Hs^s)}^2
=
\mathbb E\sup_{t\in[t_0,1]}\|X_t\|_{\Hs^s}^2.
\]
Define the map \(\mathcal T^\star\) on \(\mathbb S^2(\Hs^s)\) by
\[
(\mathcal T^\star X)_t
=
\xi
-
\int_t^1 b^\star(X_u,u)\,du
-
\int_t^1 g(u)\,d\bar W_u.
\]

A fixed point \(X=\mathcal T^\sharp X\) satisfies exactly the integral identity in \eqref{eq:reverse_integral_equation}, and any adapted continuous process satisfying that identity is a fixed point. To this end, we first show that \(\mathcal T^\star\) maps \(\mathbb S^2(\Hs^s)\) into itself. Given the inequality
\begin{equation}
    \|a+b+c\|_{\Hs^s}^2
\le
3\|a\|_{\Hs^s}^2
+
3\|b\|_{\Hs^s}^2
+
3\|c\|_{\Hs^s}^2,
\label{eq:3ineq}
\end{equation}
we obtain
\begin{equation}
    \sup_{t\in[t_0,1]}
\|(\mathcal T^\star X)_t\|_{\Hs^s}^2
\le
3\|\xi\|_{\Hs^s}^2
+
3\sup_{t\in[t_0,1]}
\left\|
\int_t^1 b^\star(X_u,u)\,du
\right\|_{\Hs^s}^2  
+
3\sup_{t\in[t_0,1]}
\left\|
\int_t^1 g(u)\,d\bar W_u
\right\|_{\Hs^s}^2 .
\label{eq:app-basic-split-wellposed}
\end{equation}
For each fixed \(t\in[t_0,1]\), the Cauchy--Schwarz inequality for Bochner integrals gives
\[
\left\|
\int_t^1 b^\star(X_u,u)\,du
\right\|_{\Hs^s}^2
\le
(1-t)\int_t^1
\|b^\star(X_u,u)\|_{\Hs^s}^2\,du .
\]
Since \(1-t\le 1-t_0\) and \([t,1]\subset[t_0,1]\), we have
\begin{equation}
    \label{eq:drift_est}
    \sup_{t\in[t_0,1]}
\left\|
\int_t^1 b^\star(X_u,u)\,du
\right\|_{\Hs^s}^2
\le
(1-t_0)\int_{t_0}^1
\|b^\star(X_u,u)\|_{\Hs^s}^2\,du .
\end{equation}
From the linear growth bound \eqref{eq:linear_growth_drift},
\begin{equation}
    \label{eq:growth_}
    \|b^\star(X_u,u)\|_{\Hs^s}^2
\le
K_s^2\bigl(1+\|X_u\|_{\Hs^s}\bigr)^2
\le
2K_s^2\bigl(1+\|X_u\|_{\Hs^s}^2\bigr),
\end{equation}
where the last inequality follows from \((a+b)^2\le 2a^2+2b^2 .\) Combining Equation~\eqref{eq:growth_} with drift estimate~\eqref{eq:drift_est}, we have
\[
\mathbb E\sup_{t\in[t_0,1]}
\left\|
\int_t^1 b^\star(X_u,u)\,du
\right\|_{\Hs^s}^2
\le
2(1-t_0)K_s^2
\int_{t_0}^1
\mathbb E\bigl(1+\|X_u\|_{\Hs^s}^2\bigr)\,du .
\]
Since
\[
\mathbb E\|X_u\|_{\Hs^s}^2
\le
\mathbb E\sup_{v\in[t_0,1]}\|X_v\|_{\Hs^s}^2
=
\|X\|_{\mathbb S^2(\Hs^s)}^2,
\]
we obtain
\[
\int_{t_0}^1
\mathbb E\bigl(1+\|X_u\|_{\Hs^s}^2\bigr)\,du
\le
(1-t_0)
\Bigl(
1+\|X\|_{\mathbb S^2(\Hs^s)}^2
\Bigr).
\]
Hence,
\[
\mathbb E\sup_{t\in[t_0,1]}
\left\|
\int_t^1 b^\star(X_u,u)\,du
\right\|_{\Hs^s}^2
\le
2(1-t_0)^2K_s^2
\Bigl(
1+\|X\|_{\mathbb S^2(\Hs^s)}^2
\Bigr).
\]
For the stochastic term, Doob's \(L^2\) maximal inequality and It\^o isometry give
\[
\mathbb E\sup_{t\in[t_0,1]}
\left\|
\int_t^1 g(u)\,d\bar W_u
\right\|_{\Hs^s}^2
\le
4\int_{t_0}^1 |g(u)|^2\,du .
\]
The drift and stochastic estimates, applied to \eqref{eq:app-basic-split-wellposed}, give
\[
\|\mathcal T^\star X\|_{\mathbb S^2(\Hs^s)}^2
\le
C\Bigl(
1+\mathbb E\|\xi\|_{\Hs^s}^2
+
\|X\|_{\mathbb S^2(\Hs^s)}^2
\Bigr).
\]
Thus \(\mathcal T^\star X\in \mathbb S^2(\Hs^s)\) when
\(X\in \mathbb S^2(\Hs^s)\).

We prove existence and uniqueness by contraction on a short interval and iteration. Let \(X,Y\in\mathbb S^2(\Hs^s)\). Then
\[
(\mathcal T^\star X)_t-(\mathcal T^\star Y)_t
=
-
\int_t^1
\bigl(
b^\star(X_u,u)-b^\star(Y_u,u)
\bigr)\,du,
\]
hence
\[
\sup_{t\in[t_0,1]}
\|(\mathcal T^\star X)_t-(\mathcal T^\star Y)_t\|_{\Hs^s}^2
\le
(1-t_0)\int_{t_0}^1
\|b^\star(X_u,u)-b^\star(Y_u,u)\|_{\Hs^s}^2\,du.
\]
Using Lipschitz condition \eqref{eq:global_lipschitz_drift}, \(\|b^\star(X_u,u)-b^\star(Y_u,u)\|_{\Hs^s}^2\),
\[
\|\mathcal T^\star X-\mathcal T^\star Y\|_{\mathbb S^2(\Hs^s)}^2
\le
(1-t_0)^2L_s^2
\|X-Y\|_{\mathbb S^2(\Hs^s)}^2.
\]
If the interval is short enough that \((1-t_0)^2L_s^2<1,\)
then \(\mathcal T^\star\) is a contraction and Banach's fixed point theorem gives a unique solution on \([t_0,1]\). On a general finite interval, one partitions \([t_0,1]\) into finitely many subintervals of this form and iterates the argument backward from time \(1\), which gives us the existence and uniqueness of an adapted continuous solution on the whole interval.

Let \(X_t^\star
=
\xi
-
\int_t^1 b^\star(X_u^\star,u)\,du
-
\int_t^1 g(u)\,d\bar W_u, \) and define
 define \(M(a)
=
\mathbb E\sup_{t\in[a,1]}\|X_t^\star\|_{\Hs^s}^2, \) for \(a\in[t_0,1]\).
Using drift estimate, the linear growth bound \eqref{eq:linear_growth_drift}, Doob's inequality, and It\^o isometry gives us
\[
M(a)
\le
C\Bigl(
1+\mathbb E\|\xi\|_{\Hs^s}^2
+
\int_a^1 |g(u)|^2\,du
+
\int_a^1 M(u)\,du
\Bigr).
\]
Gronwall's lemma in backward time gives
\[
M(t_0)
=
\mathbb E\sup_{t\in[t_0,1]}\|X_t^\star\|_{\Hs^s}^2
\le
C_s\Bigl(1+\mathbb E\|\xi\|_{\Hs^s}^2\Bigr),
\]
which proves \eqref{eq:moment_bound_reverse}.

Under \eqref{eq:derivative_bound_drift}, the operator family
\[
B_t^\Gamma
=
D_x b^\star(X_t^\star,t)
\]
is uniformly bounded on \(\Hs^s\). The linear ODE
\[
\frac{d}{dt}Z_t
=
B_t^\Gamma Z_t
\]
therefore has a unique evolution family \(\mathcal U_{\tau,t}^\Gamma\) on \(\Hs^s\). For \(Z_t=\mathcal U_{\tau,t}^\Gamma z\), we have
\[
\frac{d}{dt}\|Z_t\|_{\Hs^s}
\le
\|B_t^\Gamma\|_{\mathcal L(\Hs^s)}
\|Z_t\|_{\Hs^s}
\le
L_s'\|Z_t\|_{\Hs^s},
\]
and Gronwall's inequality gives
\[
\|Z_t\|_{\Hs^s}
\le
e^{L_s'|t-\tau|}
\|z\|_{\Hs^s}.
\]
Since \(z\in\Hs^s\) was arbitrary, this proves \eqref{eq:flow_bound_Hs}.
\end{proof}

\section{Proof of Theorem~\ref{thm:recovery_bound}}
\label{app-proof-recovery-bound}

\begin{lemma}[]
\label{lem:exact_term}
Given Assumption~\ref{ass:contraction}, let \(m_y\) and \(C_y\) be the conditional mean and covariance operator of \(X_0^\star\) given \(y\), so that
\begin{equation}
\mathbb E\left[\|m_y-x^\dagger\|_{\Hs^s}^2\right]
\le
C_{\mathrm{bias}}\,\sigma^{2\eta(r,\beta,s)}
, \quad \mathbb E\left[\operatorname{Tr}\!\bigl(L^{s/2}C_yL^{s/2}\bigr)\right]
\le
C_{\mathrm{var}}\,\sigma^{2\eta(r,\beta,s)} .
\label{eq_bias_var}
\end{equation}
Then
\begin{equation}
\mathbb E\|X_0^\star-x^\dagger\|_{\Hs^s}^2
\le
\Phi_{\mathrm{noise}}(\sigma;r,\beta,s),
\label{eq_exact_term_bound}
\end{equation}
where
\[
\Phi_{\mathrm{noise}}(\sigma;r,\beta,s)
=
C_{\mathrm{noise}}\,\sigma^{2\eta(r,\beta,s)},
\qquad
C_{\mathrm{noise}}=C_{\mathrm{bias}}+C_{\mathrm{var}}.
\]
\end{lemma}

\begin{proof}
By Assumption~\ref{ass:contraction},
\[
\mathbb E[X_0^\star\mid y]=m_y,
\qquad
\operatorname{Cov}(X_0^\star\mid y)=C_y.
\]
Define \(Z=X_0^\star-m_y.\)
Then \(Z\mid y\) has mean zero and covariance \(C_y\), and
\begin{align}
\mathbb E\!\left[\|X_0^\star-x^\dagger\|_{\Hs^s}^2\mid y\right]
&=
\mathbb E\!\left[
\|Z+(m_y-x^\dagger)\|_{\Hs^s}^2
\mid y
\right] \notag\\
&\overset{(a)}{=}
\mathbb E\!\left[\|Z\|_{\Hs^s}^2\mid y\right]
+
\|m_y-x^\dagger\|_{\Hs^s}^2
+
2\,\mathbb E\!\left[
\langle Z,m_y-x^\dagger\rangle_{\Hs^s}
\mid y
\right] \notag\\
&=
\mathbb E\!\left[\|Z\|_{\Hs^s}^2\mid y\right]
+
\|m_y-x^\dagger\|_{\Hs^s}^2,
%\label{eq_exact_expand}
\label{eq_exact_reduced}
\end{align}
where \((a)\) expands the squared Hilbert norm, and the cross product vanishes since
%\(\mathbb E[Z\mid y]=0\).The cross term vanishes because \(m_y-x^\dagger\) is deterministic under the conditional expectation and \(\mathbb E[Z\mid y]=0\), hence
\[
\mathbb E\!\left[
\langle Z,m_y-x^\dagger\rangle_{\Hs^s}
\mid y
\right]
=
\left\langle
\mathbb E[Z\mid y],\,m_y-x^\dagger
\right\rangle_{\Hs^s}
=
0.
\]

By definition of the Hilbert-scale norm,\(\|Z\|_{\Hs^s}^2
=
\|L^{s/2}Z\|_{\Hs}^2.\)
Let \(\{e_j\}_{j\ge 1}\) be an orthonormal basis of \(\Hs\). Since \(L^{s/2}C_yL^{s/2}\) is trace class, we have
\begin{align}
\mathbb E\!\left[\|Z\|_{\Hs^s}^2\mid y\right]
&=
\mathbb E\!\left[\|L^{s/2}Z\|_{\Hs}^2\mid y\right] \notag\\
&=
\mathbb E\!\left[
\sum_{j=1}^\infty
\bigl|
\langle L^{s/2}Z,e_j\rangle_{\Hs}
\bigr|^2
\mid y
\right] \notag\\
&=
\sum_{j=1}^\infty
\mathbb E\!\left[
\bigl|
\langle L^{s/2}Z,e_j\rangle_{\Hs}
\bigr|^2
\mid y
\right] \notag\\
&=
\sum_{j=1}^\infty
\mathbb E\!\left[
\langle Z,L^{s/2}e_j\rangle_{\Hs}
\langle L^{s/2}e_j,Z\rangle_{\Hs}
\mid y
\right] \notag\\
&=
\sum_{j=1}^\infty
\langle C_yL^{s/2}e_j,L^{s/2}e_j\rangle_{\Hs} \notag\\
&=
\sum_{j=1}^\infty
\langle L^{s/2}C_yL^{s/2}e_j,e_j\rangle_{\Hs} \notag\\
&=
\operatorname{Tr}\!\bigl(L^{s/2}C_yL^{s/2}\bigr).
\label{eq_exact_trace}
\end{align}
Using the law of iterated expectation, we have
\begin{align}
\mathbb E\|X_0^\star-x^\dagger\|_{\Hs^s}^2
&=
\mathbb E\Big[
\mathbb E\!\left[\|X_0^\star-x^\dagger\|_{\Hs^s}^2\mid y\right]
\Big] \notag\\
&=
\mathbb E\left[
\operatorname{Tr}\!\bigl(L^{s/2}C_yL^{s/2}\bigr)
\right]
+
\mathbb E\left[
\|m_y-x^\dagger\|_{\Hs^s}^2
\right] \notag\\
&\le
(C_{\mathrm{bias}}+C_{\mathrm{var}})\sigma^{2\eta(r,\beta,s)}.
\end{align}
\end{proof}

\begin{lemma}[]
\label{lem:score_term}
Let \(s^\star(x,t)=\nabla\log p_t(x)\). Assume that the exact posterior reverse process \(X_t^\star\) and the learned posterior reverse process \(X_t^\theta\) solve
\begin{equation}
dX_t^\star
=
\Big(
f(X_t^\star,t)
-
g(t)^2 s^\star(X_t^\star,t)
-
g(t)^2\Gamma_t\nabla_x\log p(y\mid X_t^\star)
\Big)\,dt
+
g(t)\,d\bar W_t,
\label{eq_exact_reverse}
\end{equation}
and
\begin{equation}
dX_t^\theta
=
\Big(
f(X_t^\theta,t)
-
g(t)^2 s_\theta(X_t^\theta,t)
-
g(t)^2\Gamma_t\nabla_x\log p(y\mid X_t^\theta)
\Big)\,dt
+
g(t)\,d\bar W_t,
\label{eq_learned_reverse}
\end{equation}
with the same terminal state and the same reverse Brownian motion. Assume also that
\begin{equation}
\sup_{t\in[t_0,1]}
\mathbb E
\bigl\|
s_\theta(X_t^\theta,t)-s^\star(X_t^\theta,t)
\bigr\|_{\Hs^{-q}}^2
\le
\delta_\theta^2
\label{eq_score_error_prop}
\end{equation}
for some \(q\ge 0\), and that the  evolution family \(\mathcal U_{\tau,t}^{\Gamma}\) generated by \(\bar B_t^\Gamma=\int_0^1D_xb^\star(X_t^\star+\lambda(X_t^\theta-X_t^\star),t)\,d\lambda\) satisfies
\begin{equation}
\|\mathcal U_{0,t}^{\Gamma}\|_{\mathcal L(\Hs^{-q},\Hs^s)}
\le
M_{r,q,s,\Gamma}(t),
\qquad
\int_{t_0}^1 g(t)^2 M_{r,q,s,\Gamma}(t)\,dt<\infty .
\label{eq_flow_prop}
\end{equation}
Then
\begin{equation}
\mathbb E\|X_0^\theta-X_0^\star\|_{\Hs^s}^2
\le
\Phi_{\mathrm{score}}(\delta_\theta;r,q,s,\Gamma),
\label{eq_score_term_bound}
\end{equation}
where
\[
\Phi_{\mathrm{score}}(\delta_\theta;r,q,s,\Gamma)
=
\left(
\int_{t_0}^1 g(t)^2 M_{r,q,s,\Gamma}(t)\,dt
\right)^2
\delta_\theta^2 .
\]
\end{lemma}

\begin{proof}
Define \(Z_t=X_t^\theta-X_t^\star\) and \(e_t=s_\theta(X_t^\theta,t)-s^\star(X_t^\theta,t)\). Subtracting \eqref{eq_exact_reverse} from \eqref{eq_learned_reverse} gives
\[
dZ_t
=
\Big(
f(X_t^\theta,t)-f(X_t^\star,t)
-
g(t)^2\bigl(s^\star(X_t^\theta,t)-s^\star(X_t^\star,t)\bigr)
-
g(t)^2\Gamma_t\bigl(\nabla_x\log p(y\mid X_t^\theta)-\nabla_x\log p(y\mid X_t^\star)\bigr)
-
g(t)^2 e_t
\Big)\,dt.
\]
Applying the mean-value expansion to the first three terms along the segment joining \(X_t^\star\) and \(X_t^\theta\), we obtain the operator \(\bar B_t^\Gamma\) defined in the statement such that
\begin{equation}
dZ_t
=
\bigl(\bar B_t^\Gamma Z_t-g(t)^2 e_t\bigr)\,dt,
\qquad
Z_1=0.
\label{eq_difference_equation}
\end{equation}
Equation~\eqref{eq_difference_equation} is linear in \(Z_t\) with inhomogeneous term \(-g(t)^2e_t\). Let \(\mathcal U_{\tau,t}^{\Gamma}\) be the evolution family associated with the homogeneous equation
\[
dY_t =  \bar B_t^\Gamma Y_t\,dt.
\]
The variation-of-constants formula then gives
\begin{equation}
Z_0
=
\int_{t_0}^1
\mathcal U_{0,t}^{\Gamma}\,g(t)^2 e_t\,dt .
\label{eq_duhamel_formula}
\end{equation}
Consequently,
\[
\|Z_0\|_{\Hs^s}
\le
\int_{t_0}^1
g(t)^2
\|\mathcal U_{0,t}^{\Gamma}\|_{\mathcal L(\Hs^{-q},\Hs^s)}
\|e_t\|_{\Hs^{-q}}\,dt
\le
\int_{t_0}^1
g(t)^2
M_{r,q,s,\Gamma}(t)\,
\|e_t\|_{\Hs^{-q}}\,dt .
\]
Taking expectations and applying Cauchy--Schwarz gives
\[
\mathbb E\|Z_0\|_{\Hs^s}^2
\le
\left(
\int_{t_0}^1
g(t)^2 M_{r,q,s,\Gamma}(t)\,dt
\right)^2
\sup_{t\in[t_0,1]}
\mathbb E\|e_t\|_{\Hs^{-q}}^2.
\]
Using \eqref{eq_score_error_prop}, we obtain Equation~\eqref{eq_score_term_bound}.
\end{proof}

\begin{lemma}[]
\label{lem:disc_term}
Given Assumption~\ref{ass:discretization}, we have
\begin{equation}
\mathbb E\|\widehat x-X_0^\theta\|_{\Hs^s}^2
\le
\Phi_{\mathrm{disc}}(\Delta t;r,s,\Gamma),
\label{eq_disc_term_bound}
\end{equation}
where
\[
C_{\mathrm{disc}}(r,s,\Gamma)
=
2C_{\mathrm{loc}}\,e^{2C_{\mathrm{stab}}(1-t_0)}
\]
and
\[
\Phi_{\mathrm{disc}}(\Delta t;r,s,\Gamma)
=
C_{\mathrm{disc}}(r,s,\Gamma)\,(\Delta t)^{2\kappa}.
\]
\end{lemma}

\begin{proof}
Define the grid error \(E_n=\widehat X_{t_n}-X_{t_n}^\theta.\)
Using Equation~\eqref{eq_local_error_identity},
\[
E_n
=
\Psi_n^\Delta(\widehat X_{t_{n+1}},\xi_{n+1})
-
\Psi_n^\Delta(X_{t_{n+1}}^\theta,\xi_{n+1})
-
R_n.
\]
Therefore,
\[
\|E_n\|_{\Hs^s}^2
\le
2\|\Psi_n^\Delta(\widehat X_{t_{n+1}},\xi_{n+1})-\Psi_n^\Delta(X_{t_{n+1}}^\theta,\xi_{n+1})\|_{\Hs^s}^2
+
2\|R_n\|_{\Hs^s}^2.
\]
We take conditional expectation with respect to \(\mathcal F_{t_{n+1}}\) and use Equation~\eqref{eq_stability_estimate}, then
\[
\mathbb E\|E_n\|_{\Hs^s}^2
\le
2(1+C_{\mathrm{stab}}\Delta t)\,
\mathbb E\|E_{n+1}\|_{\Hs^s}^2
+
2C_{\mathrm{loc}}\,(\Delta t)^{2\kappa+1}.
\]
Since \(E_N=0\), a discrete Gronwall argument gives
\[
\mathbb E\|E_0\|_{\Hs^s}^2
\le
2C_{\mathrm{loc}}\,e^{2C_{\mathrm{stab}}(1-t_0)}\,(\Delta t)^{2\kappa}.
\]
\end{proof}

\begin{proof}
To prove Theorem~\ref{thm:recovery_bound}, we split the reconstruction error and use inequality~\eqref{eq:3ineq}
\begin{equation}
\mathbb E\|\widehat x-x^\dagger\|_{\Hs^s}^2
\le   3\,\mathbb E\|X_0^\star-x^\dagger\|_{\Hs^s}^2 + 3\,\mathbb E\|X_0^\theta-X_0^\star\|_{\Hs^s}^2 + 3\,\mathbb E\|\widehat x-X_0^\theta\|_{\Hs^s}^2.
\label{eq:bound_thm1}
\end{equation}

Using Lemma~\ref{lem:exact_term}, Lemma~\ref{lem:score_term}, and Lemma~\ref{lem:disc_term}, and applying them to Equation~\eqref{eq:bound_thm1} proves Theorem~\ref{thm:recovery_bound}.

%We now find an upper bound for each term. The exact posterior term is treated through a bias-variance decomposition of the terminal law of the exact reverse process. The score term is treated by comparing the learned and exact reverse dynamics and propagating the score perturbation through the linearized reverse flow. The discretization term is treated through a local error bound and a mean-square stability estimate. We then combine the three bounds by a direct decomposition of the final reconstruction error.
\end{proof}

\section{Proof of Theorem~\ref{thm:recoverable_threshold}}
\label{app-proof-recoverable-threshold}

\begin{proof}
Fix \(R>0\), let \(\mathcal X_\beta(R)
=
\bigl\{
x^\dagger\in\Hs^\beta
\; ; \;
\|x^\dagger\|_{\Hs^\beta}\le R
\bigr\}.\)
Assume that \(\mathcal S_q(\delta_\theta)\) is the class of learned scores satisfying
\eqref{eq:weak_score_error}.  Since \(s<s_\star(r,\beta,q,\Gamma)\), Remark~\ref{rem:monotone} shows that \(s\) is recoverable, that is, \(\mathfrak A(r,q,s,\Gamma) < \infty\). Therefore, 
\(\Phi_{\mathrm{score}}(\delta_\theta;r,q,s,\Gamma)
=
\mathfrak A(r,q,s,\Gamma)^2\delta_\theta^2 .\)

Using Theorem~\ref{thm:recovery_bound}, for every
\(x^\dagger\in\mathcal X_\beta(R)\) and every
\(s_\theta\in\mathcal S_q(\delta_\theta)\), we can write
\[
\mathbb E\|\widehat x-x^\dagger\|_{\Hs^s}^2
\le
3\,C_{\mathrm{noise}}\sigma^{2\eta(r,\beta,s)}
+
3\,\mathfrak A(r,q,s,\Gamma)^2\delta_\theta^2
+
3\,C_{\mathrm{disc}}(r,s,\Gamma)(\Delta t)^{2\kappa}.
\]
By Assumption~\ref{ass:contraction} and the uniformity of \(C_{\mathrm{loc}}\) and \(C_{\mathrm{stab}}\) over \(\mathcal S_q(\delta_\theta)\), the constants \(C_{\mathrm{noise}}\) and \(C_{\mathrm{disc}}(r,s,\Gamma)\) are uniform over \(\mathcal X_\beta(R)\times\mathcal S_q(\delta_\theta)\). Choose
\[
C_R
\ge
3\max\{C_{\mathrm{noise}},\,1,\,C_{\mathrm{disc}}(r,s,\Gamma)\},
\]
with \(C_R\) allowed to depend on
\((R,r,\beta,q,s,\Gamma,\Sigma_\epsilon)\). Then
\[
\mathbb E\|\widehat x-x^\dagger\|_{\Hs^s}^2
\le
C_R
\Bigl(
\sigma^{2\eta(r,\beta,s)}
+
\mathfrak A(r,q,s,\Gamma)^2\delta_\theta^2
+
(\Delta t)^{2\kappa}
\Bigr),
\]
which proves the uniform recovery estimate over
\(\mathcal X_\beta(R)\times\mathcal S_q(\delta_\theta)\).
\end{proof}

\section{Proof of Theorem~\ref{thm:impossibility}}
\label{app-proof-impossibility}

\begin{proof}
Fix \(R>0\), take \(x_n^\dagger=0\in\mathcal X_\beta(R)\), and let \(c_R=1\). By \eqref{eq:sharpness} and linearity, for every \(n\) there exists \(e_n\in C([t_0,1];\Hs^s)\) with \(\sup_{t\in[t_0,1]}\|e_n(t)\|_{\Hs^{-q}}\le\delta_n\) and \(\bigl\|\int_{t_0}^1\mathcal U_{0,t}^\Gamma g(t)^2e_n(t)\,dt\bigr\|_{\Hs^s}\ge 3\). Define \(s_{\theta,n}(x,t)=s^\star(x,t)+e_n(t)\). Then \(s_{\theta,n}(x,t)-\nabla_x\log p_t(x)=e_n(t)\) for every \(x\), hence \(s_{\theta,n}\in\mathcal S_q(\delta_n)\), and \(b^{\theta,n}=b^\star-g^2e_n\) satisfies the assumptions of Proposition~\ref{prop:wellposedness} since \(e_n\) is continuous with values in \(\Hs^s\). Let \(X_t^{\star,n}\) and \(X_t^{\theta,n}\) be the exact and learned posterior reverse processes driven by the same terminal state and the same reverse Brownian motion, and define \(Z_t^n=X_t^{\theta,n}-X_t^{\star,n}\). Since \(b^\star\) is affine in \(x\), \(Z^n\) solves \(dZ_t^n=(B_t^\Gamma Z_t^n-g(t)^2e_n(t))\,dt\) with \(Z_1^n=0\) and deterministic \(B_t^\Gamma\), so the variation-of-constants formula gives the deterministic identity \(Z_0^n=\int_{t_0}^1\mathcal U_{0,t}^\Gamma g(t)^2e_n(t)\,dt\), and \(\|Z_0^n\|_{\Hs^s}\ge 3\). By Assumption~\ref{ass:contraction} with \(x_n^\dagger=0\), \(\mathbb E\|X_0^{\star,n}-x_n^\dagger\|_{\Hs^s}^2\le C_{\mathrm{noise}}\sigma_n^{2\eta(r,\beta,s)}\to 0\) since \(\eta(r,\beta,s)>0\), and by Lemma~\ref{lem:disc_term}, which applies to \(s_{\theta,n}\) by Assumption~\ref{ass:discretization}, we can choose \(\Delta t_n\downarrow 0\) with \(\mathbb E\|\widehat x_n-X_0^{\theta,n}\|_{\Hs^s}^2\le 1/n\). The triangle inequality in \(L^2(\Omega;\Hs^s)\) gives
\[
\bigl(\mathbb E\|\widehat x_n-x_n^\dagger\|_{\Hs^s}^2\bigr)^{1/2}
\ge
\|Z_0^n\|_{\Hs^s}
-\bigl(C_{\mathrm{noise}}\sigma_n^{2\eta(r,\beta,s)}\bigr)^{1/2}
-n^{-1/2},
\]
hence
\begin{equation}
\liminf_{n\to\infty}
\mathbb E\|\widehat x_n-x_n^\dagger\|_{\Hs^s}^2
\ge
9\ge c_R,
\label{eq:lower}
\end{equation}
which proves \eqref{eq:impossibility_lower_bound}.

If a uniform estimate 
\[
\mathbb E\|\widehat x-x^\dagger\|_{\Hs^s}^2
\le
C_R
\Bigl(
\sigma^{2\eta(r,\beta,s)}
+
\delta_\theta^2
+
(\Delta t)^{2\kappa}
\Bigr)
\]
holds for all \(x^\dagger\in\mathcal X_\beta(R)\) and all \(s_\theta\in\mathcal S_q(\delta_\theta)\), then applying it to the constructed sequence gives
\[
\mathbb E\|\widehat x_n-x_n^\dagger\|_{\Hs^s}^2
\le
C_R
\Bigl(
\sigma_n^{2\eta(r,\beta,s)}
+
\delta_n^2
+
(\Delta t_n)^{2\kappa}
\Bigr)
\longrightarrow 0.
\]
This contradicts the constructed bound~\eqref{eq:lower} and therefore there is no such uniform estimate can hold once
\(s>s_\star(r,\beta,q,\Gamma)\).
\end{proof}

\section{Proof of Proposition~\ref{prop:gamma-comparison}}
\label{app-proof-gamma-comparison}

\begin{proof}
    Fix \(s\in\mathcal I\) and multiply \eqref{eq:gamma-flow-comparison} by
the nonnegative weight \(g(t)^2\) and integrate over \([t_0,1]\). By the
definition \eqref{eq:recoverable}, the score-propagation functional,
we have
\[
\mathfrak A(r,q,s,\Gamma^{\mathrm{op}})
=
\int_{t_0}^1 g(t)^2M_{r,q,s,\Gamma^{\mathrm{op}}}(t)\,dt
\le
\int_{t_0}^1 g(t)^2M_{r,q,s,\Gamma^{\mathrm{iso}}}(t)\,dt
=
\mathfrak A(r,q,s,\Gamma^{\mathrm{iso}}).
\]
The integrands are nonnegative, and the two integrals are therefore well
defined in \([0,\infty]\). The comparison
\eqref{eq:gamma-propagation-comparison} thus holds without any
finiteness assumption on \(\mathfrak A(r,q,s,\Gamma^{\mathrm{iso}})\). By
\eqref{eq_phi_score}, the score 
\(\Phi_{\mathrm{score}}(\delta_\theta;r,q,s,\Gamma)=\mathfrak A(r,q,s,\Gamma)^2\delta_\theta^2\)
is a nondecreasing function of \(\mathfrak A\) on \([0,\infty]\), and
\eqref{eq:gamma-score-comparison} is a direct consequence of
\eqref{eq:gamma-propagation-comparison}.

The strict inequalities of Proposition~\ref{prop:gamma-comparison}
require \(\mathfrak A(r,q,s,\Gamma^{\mathrm{iso}})<\infty\). Fix
\(s\in\mathcal I\) under this condition, and let \(E\subset[t_0,1]\) be
a set of positive \(g(t)^2\,dt\)-measure on which the envelope of
\(\Gamma^{\mathrm{op}}\) lies strictly below the envelope of
\(\Gamma^{\mathrm{iso}}\). The first part of the proof gives
\(\mathfrak A(r,q,s,\Gamma^{\mathrm{op}})\le\mathfrak A(r,q,s,\Gamma^{\mathrm{iso}})<\infty\).
The two envelopes are thus finite \(g(t)^2\,dt\)-almost everywhere, and
the difference of the two functionals satisfies
\begin{align*}
\mathfrak A(r,q,s,\Gamma^{\mathrm{iso}})-\mathfrak A(r,q,s,\Gamma^{\mathrm{op}})
&=
\int_{t_0}^1 g(t)^2\bigl(M_{r,q,s,\Gamma^{\mathrm{iso}}}(t)-M_{r,q,s,\Gamma^{\mathrm{op}}}(t)\bigr)\,dt
\\
&\ge
\int_E g(t)^2\bigl(M_{r,q,s,\Gamma^{\mathrm{iso}}}(t)-M_{r,q,s,\Gamma^{\mathrm{op}}}(t)\bigr)\,dt
>0.
\end{align*}
The nonnegativity of the integrand on \([t_0,1]\) justifies the first
inequality, and the strict positivity of the integrand on the set \(E\)
of positive measure justifies the second. Consequently,
\(0\le\mathfrak A(r,q,s,\Gamma^{\mathrm{op}})<\mathfrak A(r,q,s,\Gamma^{\mathrm{iso}})<\infty\).
For \(\delta_\theta>0\), the map \(u\mapsto u^2\delta_\theta^2\) is
strictly increasing on \([0,\infty)\), and the strict inequality between
the two functionals carries over to the score terms in
\eqref{eq:gamma-score-comparison}.
\end{proof}

%%%%%%%%%%%%%%%%%%%%%%%%%%%%%%%%%%%%%%%%%%%%%%%%%%%%%%%%%%%%

%\newpage
%\input{checklist.tex}

\end{document}